\documentclass[11pt,a4paper]{article}
\usepackage[margin=0.75in]{geometry}
\usepackage[utf8]{inputenc}
\usepackage{cmap}
\usepackage[T1]{fontenc}
\usepackage{lmodern}
\usepackage[square, numbers]{natbib}   
\usepackage[colorlinks=true]{hyperref}
\usepackage{amssymb, amsthm, amsmath, amsfonts}
\usepackage{algorithm}
\usepackage{algpseudocode}
\usepackage{authblk, dsfont, mathrsfs, color, bm, enumitem, mathtools}
\usepackage{subfigure, graphicx, transparent}
\usepackage{dirtytalk}
\graphicspath{{figures/}}
\usepackage[dvipsnames]{xcolor}
\usepackage{comment}
\usepackage{tabularx}
\usepackage{booktabs}
\hypersetup{colorlinks, linkcolor={red!80!black}, citecolor={blue!80!black}, urlcolor={ForestGreen}}
\theoremstyle{definition}
\newtheorem{theorem}{Theorem}[section]

\newtheorem{proposition}[theorem]{Proposition}
\newtheorem{corollary}[theorem]{Corollary}
\newtheorem{lemma}[theorem]{Lemma}
\newtheorem{remark}[theorem]{Remark}

\DeclareUnicodeCharacter{211D}{\ensuremath{\mathbb R}}

\newcommand{\geqm}{\succcurlyeq}

\newcommand{\rb}{\mathbb{R}}

\def \card{ { \rm Card }}

\DeclarePairedDelimiterX{\inner}[2]{\langle}{\rangle}{#1, #2}

\title{Minimum Volume Conformal Sets for Multivariate Regression}

\author[1]{Sacha Braun}
\author[2]{Liviu Aolaritei}
\author[$1, 2$]{Michael I. Jordan}
\author[1]{Francis Bach}
\affil[1]{Sierra team, Inria Paris, France \protect\\
\texttt{\{sacha.braun, francis.bach\}@inria.fr}}
\affil[2]{Department of Electrical Engineering and Computer Sciences, UC Berkeley, USA \protect\\ \texttt{liviu.aolaritei@berkeley.edu, jordan@cs.berkeley.edu}}

\date{}

\begin{document}

\maketitle

%------------------------------------------------------------------------

\begin{abstract}
{\color{black}Conformal prediction provides a principled framework for constructing predictive sets with finite-sample validity. While much of the focus has been on univariate response variables, existing multivariate methods either impose rigid geometric assumptions or rely on flexible but computationally expensive approaches that do not explicitly optimize prediction set volume. We propose an optimization-driven framework based on a novel loss function that directly learns minimum-volume covering sets while ensuring valid coverage. This formulation naturally induces a new nonconformity score for conformal prediction, which adapts to the residual distribution and covariates. Our approach optimizes over prediction sets defined by arbitrary norm balls—including single and multi-norm formulations. Additionally, by jointly optimizing both the predictive model and predictive uncertainty, we obtain prediction sets that are tight, informative, and computationally efficient, as demonstrated in our experiments on real-world datasets.}
\end{abstract}

%------------------------------------------------------------------------

\section{Introduction}
\label{sec:intro}

In predictive modeling, quantifying uncertainty is often as crucial as making accurate predictions. Traditional point estimates provide limited insight into predictive accuracy, whereas prediction sets offer a more robust alternative by identifying regions that contain the true outcome with high probability. Conformal prediction \cite{vovk2005algorithmic, shafer2008tutorial,angelopoulos2024theoretical} provides a model-agnostic framework for constructing such sets with finite-sample validity, ensuring that the true response is captured at least $1 - \alpha$ fraction of the time without requiring strong distributional assumptions.

In the setting of univariate regression, conformal prediction can produce prediction intervals that adapt to heteroskedasticity in the data. Quantile regression, optimized using the pinball loss, is a common approach for learning such intervals \cite{romano2019conformalized}. Extending these ideas to multivariate regression, however, where the response is vector-valued, introduces significant challenges. A straightforward extension—constructing Cartesian products of marginal intervals \cite{neeven2018conformal}—fails to account for dependencies across dimensions, resulting in overly conservative and inefficient prediction sets. Instead, prediction sets must be structured to adapt to the joint distribution of the residuals, balancing validity, efficiency, and flexibility.

More structured alternatives, such as ellipsoidal prediction sets shaped by empirical residual dependencies \cite{johnstone2021conformal, messoudi2022ellipsoidal}, incorporate covariance information but assume elliptical symmetry, limiting their adaptability to more complex distributions. More flexible approaches attempt to model joint dependencies explicitly. Copula-based methods \cite{messoudi2021copula, sun2022copula} and density estimators \cite{izbicki2022cd, wang2023probabilistic} relax geometric constraints but often require accurate joint distribution estimation, which is computationally expensive in high dimensions, {\color{black}and typically leads to high estimation variance}. Other methods, including optimal transport \cite{klein2025multivariate, thurin2025optimal} and quantile region estimation \cite{feldman2023calibrated}, allow for nonconvex and multimodal structures, improving adaptability. However, these approaches typically lack an explicit volume minimization criterion, leading to unnecessarily large sets, and often involve computationally expensive procedures that scale poorly in high dimensions. These methods, along with other approaches, will be reviewed in more detail in Section~\ref{subsec:related:work} on related work.

To address these limitations, we introduce a framework for constructing multivariate conformal prediction sets that minimize volume while maintaining valid coverage. Rather than imposing a fixed structure, our approach learns the optimal shape of the prediction set by optimizing over flexible geometric representations, including adaptive norm-based formulations. This optimization extends beyond the prediction set itself, as we jointly learn the predictive model together with the prediction set, ensuring that the predictor is aligned with the minimum-volume criterion. By jointly optimizing over both the predictor and the uncertainty model, we obtain tighter prediction sets that better adapt to the underlying data distribution. Finally, we conformalize the learned sets to achieve finite-sample coverage guarantees. These methodological innovations result in a scalable, data-driven approach that improves efficiency and adaptivity while preserving coverage guarantees.

\paragraph{Contributions.} Specifically, our contributions can be summarized as follows:
\begin{itemize}
    \item[1.] \textbf{Minimum-volume covering sets (MVCS).} We introduce a general optimization-based framework for constructing minimum-volume sets that contain a prescribed fraction of a given dataset in $\mathbb{R}^k$. This extends beyond standard approaches by allowing arbitrary {\color{black}norm balls}, including data-driven norms that adapt to the geometry of the data. We reformulate the problem as a structured nonconvex optimization, providing both a difference-of-convex (DC) formulation and a convex relaxation for efficient computation. By restricting to $p$-norms, with $p \in (0,\infty)$, we further enable automatic selection of the optimal norm, and we extend this framework to multi-norm formulations, yielding highly flexible and adaptive prediction sets. {\color{black}We illustrate this in Figure~\ref{figure:intro}.}

    \item[2.] \textbf{Supervised learning with adaptive prediction sets.} We extend our MVCS framework to supervised learning, where prediction sets are constructed for a learned predictor $f_\theta$ with parameter $\theta \in \Theta$ rather than solely from residuals. Unlike standard conformal methods, which construct prediction sets post hoc, we propose an integrated optimization scheme that jointly learns the predictor, the norm structure, and the transformation function that scales the uncertainty set. {\color{black}This is achieved by introducing a \emph{novel loss function} that learns minimum-volume covering sets (MVCSs) while ensuring valid coverage.} Since the problem is nonconvex, we develop an iterative optimization scheme to jointly optimize over all parameters.

    \item[3.] \textbf{Conformalized minimum-volume prediction sets.} To ensure valid finite-sample coverage, we integrate our framework with conformal prediction, leveraging a separate calibration set to rescale the learned minimum-volume sets. This approach preserves the adaptive shape of the sets while ensuring rigorous coverage guarantees. The conformalization procedure is computationally efficient and applies seamlessly to any of our formulations—fixed norm, single learned $p$-norm, or multiple norms—making the method broadly applicable.
\end{itemize}

By bridging conformal prediction and volume optimization, our method provides a principled, data-driven framework for constructing valid, adaptive, and minimal-volume prediction sets in multivariate regression.

\begin{figure}[t]
    \center
    \subfigure{\includegraphics[width=\columnwidth]{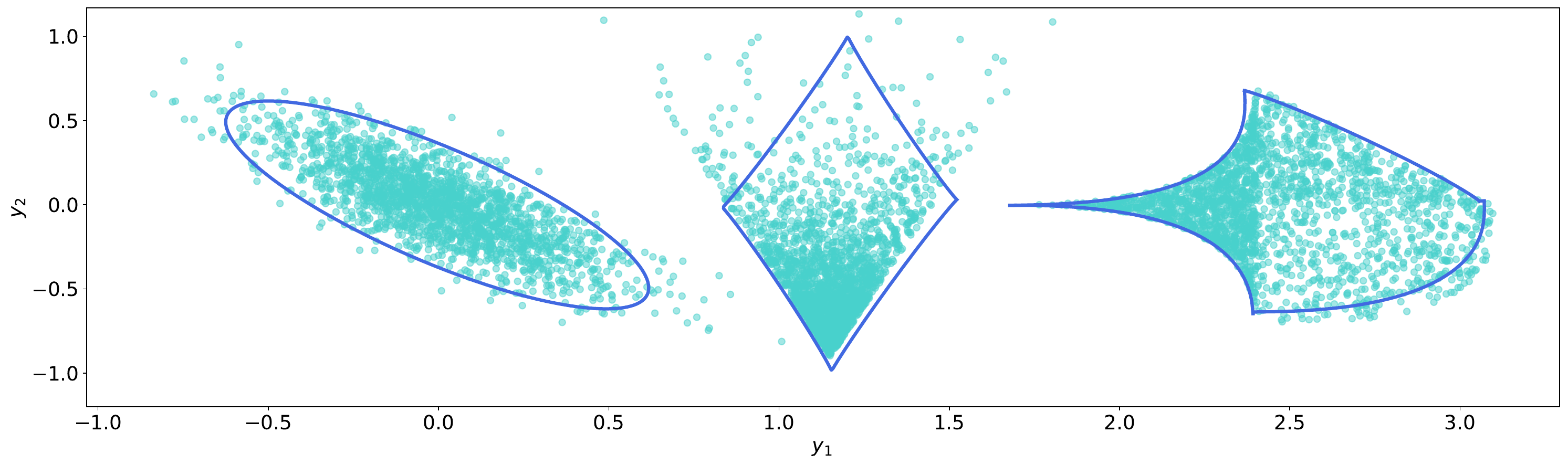}}
    \vspace{-10mm}
    \caption{MVCSs obtained using the first-order optimization strategy described in Section~\ref{subsec:p:norms} with coverage level $0.95$. The left figure is a single-norm MVCS fit with $p = 1.98$ for Gaussian data, the middle one is a single-norm MVCS fit with $p = 0.94$ for exponential data, while the right one is a multi-norm MVCS optimized over four different norms, yielding region-specific values of $p = 0.50, 0.49, 2.52, 1.11$, for a combination of a uniform distribution over three unit balls with norms $0.5$, $3$ and $1$.}
    \label{figure:intro}
\end{figure}

%------------------------------------------------------------------------

\subsection{Related work}
\label{subsec:related:work}

Several approaches have been proposed to construct valid prediction sets while balancing flexibility and efficiency \cite{dheur2025multi}. We review key contributions in this area, highlighting their strengths and limitations.

\medskip

\noindent\emph{Hyperrectangles}.
One of the earliest approaches constructs prediction sets as Cartesian products of marginal intervals, resulting in hyperrectangles \cite{neeven2018conformal}. While straightforward and computationally efficient, this method fails to capture dependencies among response variables, often leading to overly conservative prediction regions.

\medskip

\noindent\emph{Ellipsoidal prediction sets}.
Ellipsoidal prediction sets provide a more structured alternative, using empirical covariance matrices to shape prediction regions. The work \cite{johnstone2021conformal} introduced a method that constructs a single, global ellipsoidal region using the sample covariance matrix. Building on this, \cite{messoudi2022ellipsoidal} incorporated k-nearest neighbors to estimate local covariance, blending it with the global covariance matrix to improve adaptability to local structure. Finally, \cite{henderson2024adaptive} further extended these methods by developing a rigorous mathematical framework, establishing theoretical guarantees, and leveraging empirical covariance matrices for calibration and efficiency. However, while these approaches can yield tighter prediction sets in some cases, they remain restricted to convex shapes and assume an underlying elliptical structure, limiting their flexibility in capturing more complex distributions.

\medskip

\noindent\emph{Convex templates}.
To address multimodality, \cite{tumu2024multi} introduced a clustering-based method that fits separate convex templates—such as convex hulls, hyperrectangles, and ellipsoids—to each cluster. This improves flexibility in capturing diverse residual distributions while still maintaining convexity. {\color{black}Additionally, although not directly focused on prediction sets, related work has studied minimum-volume enclosing polytopes \cite{panigrahy2004minimum,saha2011new}.}

\medskip

\noindent\emph{Copula-based approaches}.
Rather than relying on fixed geometric assumptions, copula-based methods explicitly model dependencies in the response distribution. The works \cite{messoudi2021copula} and \cite{sun2022copula} leveraged copulas to construct (joint) prediction sets, offering a more data-driven approach that avoids strong parametric assumptions. These methods are inherently adaptive, as they adjust prediction regions based on the estimated dependence structure of the data. However, their effectiveness depends on accurate copula estimation, which can be challenging in high dimensions.

\medskip

\noindent\emph{Density-based and sampling-based approaches}.
Another approach to constructing prediction sets involves directly estimating the conditional distribution of the response. In this direction, \cite{izbicki2022cd} used conditional density estimation to derive predictive sets, leveraging density estimators to approximate the distribution of residuals. Alternatively, \cite{wang2023probabilistic} employed generative models to sample from the predictive distribution, capturing complex dependencies that density estimation alone might not fully capture. Moreover, \cite{plassier2024probabilistic} introduced a hybrid framework that combines both density estimation and generative modeling, improving the flexibility of predictive set construction, particularly in cases with high heteroskedasticity. {\color{black}A related line of work considers density-based quantile methods, where quantile regions are estimated using conditional density level sets \cite{camehl2024superlevel}.}

\medskip

\noindent\emph{Latent-space quantile methods}.
Nonconvex prediction sets can be constructed by mapping conditional distributions into a latent space where level sets remain convex, then transforming them back into the original space. This approach, proposed by \cite{feldman2023calibrated}, combines directional quantile regression with conditional variational autoencoders to effectively capture complex distributional features.

\medskip

\noindent\emph{Optimal transport-based methods}.
Optimal transport has been used to construct multivariate conformal prediction sets by defining meaningful rankings in multidimensional spaces. In this direction, \cite{thurin2025optimal} built upon the developments in \cite{chernozhukov2017monge, hallin2021distribution} to map multivariate residuals onto a uniform distribution using transport maps, enabling a structured approach to defining multivariate nonconformity scores. Concomitantly, \cite{klein2025multivariate} further improved computational efficiency by leveraging differentiable transport maps and entropic regularization techniques, facilitating scalable nonconformity score computations. While these methods provide a principled approach to multivariate prediction regions, they require solving transport optimization problems, which are computationally demanding in high dimensions. {\color{black}Finally, transport-based quantile methods, such as vector quantile regression \cite{carlier2016vector} and center-outward quantiles \cite{del2024nonparametric}, leverage optimal transport maps to define multivariate quantile regions. While these methods do not construct conformal prediction sets, they provide an alternative geometric approach to ordering multivariate data and characterizing uncertainty.}

\medskip

\noindent\emph{Extensions to functional and distributional data}.
Conformal prediction has also been extended beyond finite-dimensional regression to handle more complex data structures. One major direction is functional data, where \cite{lei2015conformal} introduced a conformal framework, later refined by \cite{diquigiovanni2021conformal}, to improve adaptivity for structured functional domains. In the context of distributional regression, \cite{chernozhukov2021distributional} developed conformal methods that provide coverage guarantees in probability spaces. More broadly, \cite{kuchibhotla2020exchangeability} formulated a general conformal prediction framework for arbitrary metric spaces, expanding its applicability to a wide range of structured data types.

\medskip

\noindent\emph{Volume-minimizing methods}.
Among prior works, \cite{henderson2024adaptive} and \cite{tumu2024multi} are the most closely related to our method, as both focus explicitly on \emph{minimizing the volume of the prediction set}. However, \cite{tumu2024multi} restricts prediction regions to convex shapes and employs heuristic clustering algorithms to adaptively partition the data, while \cite{henderson2024adaptive} optimizes volume under the assumption of elliptical distributions, relying on empirical covariance estimation. {\color{black}Additionally, although not directly focused on prediction sets, a closely related line of research has investigated the problem of finding the minimum-volume enclosing ellipsoid for a given set of points \cite{sun2004computation,todd2016minimum,todd2007khachiyan,tao2017global,bowman2023computing,rosa2022computing,kumar2005minimum,harris2024efficient,mittal2022finding}.} Our method generalizes these approaches by optimizing prediction regions over arbitrary norms, removing restrictive convexity assumptions while maintaining exact finite-sample coverage.

\medskip

\noindent\emph{Local adaptivity}.
Beyond their structural differences, some of these methods are designed to be adaptive, adjusting prediction sets based on the underlying data distribution \cite{thurin2025optimal, messoudi2022ellipsoidal, diquigiovanni2021conformal, chernozhukov2021distributional, kuleshov2018conformal, messoudi2020conformal, messoudi2021copula, kuchibhotla2020exchangeability, izbicki2022cd, wang2023probabilistic, plassier2024conditionally, tumu2024multi, feldman2023calibrated, dheur2025multi}, while others provide theoretical guarantees of asymptotic conditional coverage \cite{thurin2025optimal, chernozhukov2021distributional, izbicki2022cd, plassier2024conditionally}. Our method belongs to the class of adaptive approaches.

%------------------------------------------------------------------------

\subsection{Preliminaries on conformal prediction}
\label{subsec:preliminaries:CP}

We focus on \emph{split conformal prediction}, one of the most widely used methods for constructing valid prediction sets due to its simplicity and computational efficiency \cite{papadopoulos2002inductive, lei2018distribution}. Given a dataset $\mathcal{D} = \{(X_i, Y_i)\}_{i=1}^{N} \sim \mathbb P$, where each $X_i \in \mathcal{X} \subseteq \mathbb{R}^d$ represents a covariate vector, with $d$ components that we refer to as \emph{features}, and each $Y_i \in \mathcal{Y} \subseteq \mathbb{R}^k$ represents a multivariate response, the goal is to construct a prediction region $C(X)$ for a new test covariate $X$ such that it contains the true response $Y$ with probability at least $1 - \alpha$.

To achieve this, a predictive model $f: \mathcal{X} \to \mathcal{Y}$ is first fitted using a portion of the data, referred to as the \emph{training set}. The remaining $n$ data points, known as the \emph{calibration set}, are then used to quantify the uncertainty of the model’s predictions by adjusting the prediction sets to achieve valid coverage. Specifically, conformal prediction constructs a prediction region that satisfies the finite-sample coverage guarantee
\begin{align}
\label{eq:basic:CP:guarantee}
    \mathrm{Prob} \left\{ Y_{n+1} \in C(X_{n+1}) \right\} \geq 1 - \alpha.
\end{align}
The probability is taken over both the randomness in the test point $(X_{n+1}, Y_{n+1})$ and the calibration set, under the assumption that all data points are exchangeable but with no further distributional assumptions. 

A key component in achieving this guarantee is the choice of a \emph{nonconformity score} function $s: \mathcal{X} \times \mathcal{Y} \to \mathbb{R}$, which measures how atypical an observed response $Y$ is relative to the model’s prediction for a given covariate vector $X$. A standard choice in multivariate regression is the norm of the residual $s(X, Y) = \| f(X) - Y \|$, where $\| \cdot \|$ is typically the Euclidean norm, though other norms can be considered to better capture the structure of the residuals.

Once nonconformity scores are computed for all calibration points, a threshold is determined by estimating the empirical $(1-\alpha)$-quantile of the empirical distribution of these scores, adjusted with a \emph{finite-sample correction}:
\begin{align*}
    \widehat{q}_\alpha = \mathrm{Quantile}\left(\frac{\lceil (1-\alpha)(n+1) \rceil}{n}; \frac{1}{n}\sum_{i=1}^n \delta_{s(X_i, Y_i)} \right).
\end{align*}
This threshold defines the final prediction region, which for a test covariate $X_{n+1}$ is given by:
\begin{align*}
    C(X_{n+1}) = \left\{ y \in \mathcal{Y} \mid s(X_{n+1}, y) \leq \widehat{q}_\alpha \right\}.
\end{align*}
By construction, this region satisfies the finite-sample coverage guarantee in \eqref{eq:basic:CP:guarantee}, as the threshold is determined by a rank statistic: the empirical quantile is chosen so that at least $\lceil (1-\alpha)(n+1) \rceil$ calibration points have nonconformity scores below it. Since the test point is exchangeable with the calibration points, its rank among them follows a uniform distribution, ensuring valid coverage.

One of the main advantages of split conformal prediction is its efficiency: since the predictive model is trained only once, prediction regions can be rapidly computed for new test points. However, a key limitation is that a portion of the dataset must be allocated for calibration rather than model training, which may lead to a slight reduction in predictive performance compared to methods that utilize the full dataset. Despite this, its ease of implementation, broad applicability, and strong theoretical guarantees make it a widely adopted tool for uncertainty quantification in regression.

\paragraph{Notation}
Given a set $\{a_i\}_{i=1}^n \subset \rb$, we denote by $\sigma_r\{a_i\}$ the $r$-th largest element of the set, and by $\overline{\sigma}_r\{z_i\}$ the average of the $r$ largest values of the set, i.e., $\overline{\sigma}_r\{z_i\} := \frac{1}{r} \sum_{j=1}^r \sigma_j \{z_i\}$. We denote by $\mathds{1}\{x \in A\}$ the indicator function, equal to $1$ if $x \in A$ and to $0$ if $x \notin A$, for some set $A$. We denote by $\mathrm{SO}(k)$ the group of rotation matrices in $\mathbb{R}^{k \times k}$, given by $\mathrm{SO}(k) \coloneqq \{R \in \mathbb{R}^{k \times k} \mid R^\top R = I, \; \det(R) = 1\}$. Moreover, we denote by $\mathrm{Diag}(k)$ the space of diagonal matrices in $\mathbb R^{k \times k}$. We assume that if the objective function of a minimization (maximization) problem is expressed as the difference of two terms, both evaluating to $\infty$, then the objective function value is interpreted as $\infty$ ($-\infty$). This aligns with the extended arithmetic rules in \cite{rockafellar2009variational}.

%------------------------------------------------------------------------
%------------------------------------------------------------------------
%------------------------------------------------------------------------
%------------------------------------------------------------------------

\section{Minimum-Volume Covering Set}
\label{sec:min:vol:cov:shape}

In this section, we consider the general problem of finding the minimum-volume set that contains a fraction $1-\alpha$ of a given set of points $\{y_1, \dots, y_n\}$ in $\mathbb{R}^k$, for some $\alpha \in (0,1)$. In Section~\ref{subsec:application:regression}, we show how this formulation extends naturally to regression, where we observe pairs $\{(x_i, y_i)\}_{i=1}^n$ and construct adaptive prediction regions that account for multivariate uncertainty. To define these sets, we consider arbitrary norms $\|\cdot\|$ on $\mathbb{R}^k$. Formally, given $\|\cdot\|$, we define:
\begin{align}
\label{eq:ball}
    \mathbb{B}(\|\cdot\|, M, \mu) := \{ y \in \mathbb{R}^k \mid \|M (y - \mu)\| \leq 1 \}.
\end{align}
Here, $\mu \in \mathbb{R}^k$ specifies the center, while $M \in \mathbb{R}^{k \times k}$, with $M \geqm 0$, determines its shape. These sets serve as our fundamental prediction regions, and we aim to optimize $M$ and $\mu$ to minimize their volume.

A notable case arises when $\|\cdot\|$ is the $p$-norm, denoted $\|\cdot\|_p$, for some $p \in (0, \infty)$. In this setting, we use the notation $\mathbb{B}(p, M, \mu)$. Beyond optimizing $M$ and $\mu$, we also learn $p$, allowing the norm structure to be selected adaptively based on the data. More generally, combining multiple $p$-norms allows for asymmetric prediction sets, overcoming the symmetry constraints of individual norms. This enables better adaptation to the geometry of the data while ensuring the required $1-\alpha$ coverage. A formal definition of this multi-norm formulation is presented in Section~\ref{subsec:p:norms}.

In the rest of this section, we first establish a general framework for computing the minimum-volume covering set by optimizing $M$ and $\mu$ for a given norm, as introduced in Section~\ref{subsec:arbitrary:norms}. Since this problem is inherently nonconvex, we propose a difference-of-convex (DC) reformulation and a convex relaxation to enable efficient computation. We then focus on $p$-norms in Section~\ref{subsec:p:norms}, jointly optimizing over $M$, $\mu$, and $p \in (0, \infty)$, allowing the optimization process to automatically select the most suitable $p$\textcolor{black}{, leveraging gradient-based methods for nonconvex optimization}. Moreover, within the same section, we extend this framework to multiple norms, enabling the construction of more flexible, asymmetric prediction sets. Finally, in Section~\ref{subsec:application:regression} we employ this framework in the context of supervised learning.

%------------------------------------------------------------------------

\subsection{Arbitrary norm-based sets}
\label{subsec:arbitrary:norms}

Given a norm $\|\cdot\|$ on $\mathbb{R}^k$, our objective is to find the minimum-volume norm-based set that satisfies a prescribed coverage constraint. This leads to the optimization problem:
\begin{align}
\label{eq:min:vol}
\begin{split}
    \min \quad & \textrm{Vol}(\mathbb{B}(\|\cdot\|, M, \mu)) \\
    \mathrm{s.t.} \quad & M \geqm 0, \; \mu \in \mathbb{R}^k, \\
    & \mathrm{Card} \left\{ i \in [n] \mid \|M (y_i - \mu)\| \leq 1 \right\} \geq n - r + 1.
\end{split}
\end{align}
Here, the constraint ensures that the set contains at least $n - r + 1$ of the given points $\{ y_1, \dots, y_n \}$. The quantity $r$ is chosen such that this corresponds to a fraction $1 - \alpha$ of the total points. The volume of a norm-based set is given by
\begin{align}
\label{eq:volume}
    \textrm{Vol}(\mathbb{B}(\|\cdot\|, M, \mu)) = \lambda(B_{\|\cdot\|}(1)) \cdot \det (M)^{-1},
\end{align}
where $\mathbb B_{\|\cdot\|}(1) = \{ y \in \mathbb{R}^k \mid \|y\| \leq 1 \}$ is the unit ball under $\|\cdot\|$, and $\lambda$ denotes the Lebesgue measure in $\mathbb{R}^k$. Since we do not optimize over the norm $\|\cdot\|$ in this section, the term $\lambda(B_{\|\cdot\|}(1))$ remains fixed, and we can simplify the objective function by minimizing $-\log \det(M)$. This reformulation preserves the essential structure of the problem while removing unnecessary constants. We will return to the full volume formulation in Section~\ref{subsec:p:norms}, where we optimize over the norm, making the dependence on $\lambda(B_{\|\cdot\|}(1))$ relevant. Under these considerations, the problem reduces to
\begin{align}
\label{eq:min:vol:problem}
\begin{split}
    \min \quad & -\log \det(M) \\
    \mathrm{s.t.} \quad & M \geqm 0, \; \mu \in \mathbb{R}^k, \\
    & \mathrm{Card} \left\{ i \in [n] \mid \|M (y_i - \mu)\| \leq 1 \right\} \geq n - r + 1.
\end{split}
\end{align}

Throughout this section, we refer to problem~\eqref{eq:min:vol:problem} as a \emph{minimum-volume covering set} (MVCS) problem. While the objective function $-\log \det(M)$ is convex for $M \geqm 0$, the problem remains nonconvex due to the cardinality constraint, which introduces a combinatorial structure to the feasible set. This problem is known to be NP hard (see \cite[Proposition~1]{ahmadi2014robust}). The following proposition, {\color{black}based on a homogeneity argument}, provides an exact reformulation that enables both convex relaxations and exact reformulations via difference-of-convex (DC) programming. \textcolor{black}{In Section~\ref{subsec:p:norms} we also formulate this loss in a way that allows the use of first-order optimization strategies.}

\begin{proposition}[Reformulation of the MVCS problem]
\label{prop:min:vol:exact}
The MVCS problem~\eqref{eq:min:vol:problem} has the same optimal value as the following optimization problem:
\begin{align}
\label{eq:min:vol:exact}
\begin{split}
    \min \quad & - \log \det(\Lambda) + \sigma_r\left\{ \|\Lambda y_i + \eta \| \right\} \\
    \mathrm{s.t.} \quad & \Lambda \geqm 0,\; \eta \in \rb^k,
\end{split}
\end{align}
where $\sigma_r\left\{ \|\Lambda y_i + \eta \| \right\}$ denotes the $r$-th largest value among $\{\|\Lambda y_i + \eta \|\}_{i=1}^n$. 
Moreover, the optimal pair $(M^\star,\mu^\star)$ in \eqref{eq:min:vol:problem} can be recovered from the optimal solution $(\Lambda^\star,\eta^\star)$ of \eqref{eq:min:vol:exact} as
\begin{align*}
    M^\star = \sigma_r\left\{ \|\Lambda^\star y_i + \eta^\star\|_p \right\}^{-1} \Lambda^\star, \quad
    \mu^\star = -(\Lambda^\star)^{-1} \eta^\star.
\end{align*}
\end{proposition}
\begin{proof}
Using the change of coordinates $\Lambda \coloneqq \nu M$, for $\nu \geq 0$, and $\eta \coloneqq -\Lambda \mu$, we have that problem~\eqref{eq:min:vol:problem} is equivalent to
\begin{align}
\label{eq:min:vol:exact:temp:1}
\begin{split}
    \min \quad & - \log \det (\Lambda) + k \log \nu \\\;
    \mathrm{s.t.} \quad & \Lambda \geqm 0,\; \mu \in \rb^k, \nu \geq 0 \\
    &\card \left\{ i \in [n] \mid \|\Lambda y_i + \eta\|\leq \nu \right\} \geq n - r + 1.
\end{split}
\end{align}
Now, due the cardinality constraint, for given $\Lambda$ and $\eta$, the optimal $\nu^\star(\Lambda, \eta)$ is equal to the $(n-r+1)$-th smallest value in the set $\left\{\|\Lambda y_i + \eta\|\right\}_{i=1}^n$, or, equivalently, the $r$-th largest value, i.e.,
\begin{align*}
    \nu^\star(\Lambda, \eta) = \sigma_r\left\{ \|\Lambda y_i + \eta\| \right\}.
\end{align*}
Using this, we have that problem \eqref{eq:min:vol:exact:temp:1} is equivalent to the following optimization problem:
\begin{align}
\label{eq:min:vol:exact:temp:2}
\begin{split}
    \min \quad & - \log \det (\Lambda) + k \log\sigma_r\left\{ \|\Lambda y_i + \eta\| \right\} \\
    \mathrm{s.t.} \quad & \Lambda \geqm 0,\; \mu \in \rb^k.
\end{split}
\end{align}
We will now prove that formulation \eqref{eq:min:vol:exact:temp:2} is equivalent to \eqref{eq:min:vol:exact} in the statement of the proposition. To see this, notice that \eqref{eq:min:vol:exact} is equivalent to 
\begin{align*}
\begin{split}
    \min \quad & - \log \det (\zeta \Lambda) + \sigma_r\left\{ \|\zeta( \Lambda y_i + \eta)\| \right\} \\
    \mathrm{s.t.} \quad & \Lambda \geqm 0,\; \eta \in \rb^k, \zeta \geq 0, \\
\end{split}
\end{align*}
whose objective function can be rewritten as $- k \log \zeta - \log \det (\Lambda) +  \zeta \sigma_{r} \left\{ \|\Lambda y_i + \eta\|  \right\}$. Now, since the objective function is convex and differentiable in $\zeta$, we can minimize explicitly over $\zeta$ and obtain
\begin{align*}
    \zeta^\star = \frac{k}{\sigma_{r} \left\{ \|\Lambda y_i + \eta\|  \right\}},
\end{align*}
which is nonnegative, as required. Plugging this back into the objective function, we obtain
\begin{align*}
\begin{split}
    \min \quad &- \log \det (\Lambda) + k  \log \sigma_{p} \left\{ \|\Lambda y_i + \eta\|  \right\}  +k - k \log k
    \\ \mathrm{s.t.} \quad & \Lambda \geqm 0,\; \eta \in \rb^k.
\end{split}
\end{align*}
which is equivalent to \eqref{eq:min:vol:exact:temp:2}. This concludes the proof.
\end{proof}

Following Proposition~\ref{prop:min:vol:exact}, the MVCS is given by
\begin{align*}
    \mathbb{B}(\|\cdot\|, M^\star, \mu^\star) := \{ y \in \mathbb{R}^k \mid \|M^\star (y - \mu^\star)\| \leq 1 \} = \left\{ y \in \mathbb R^k  \mid  \| \Lambda^\star y + \eta^\star \| \leq \sigma_r\left\{ \|\Lambda^\star y_i + \eta^\star\| \right\} \right\}.
\end{align*}
The reformulation~\eqref{eq:min:vol:exact} recasts the original MVCS problem into an optimization framework where volume minimization, captured by the determinant term, is explicitly balanced against the coverage constraint, enforced through $\sigma_r\{\|\Lambda y_i + \eta\|\}$. While the problem remains nonconvex, this structured formulation makes the role of each component explicit, facilitating further analysis and the development of tractable approximations.

The nonconvexity of problem~\eqref{eq:min:vol:exact} arises from the function $\sigma_r$, which selects the $r$-th largest value among the nonconformity scores. However, despite its nonconvex nature, the problem admits a structured decomposition into a \emph{difference-of-convex (DC) form}. Specifically, we express the objective function in problem~\eqref{eq:min:vol:exact} as the difference of two convex functions,
\begin{align}
\label{eq:obj:funct:min:vol}
    \underbrace{- \log \det \Lambda + r \overline{\sigma}_r\left\{ \|\Lambda y_i + \eta \| \right\}}_{f(\Lambda,\eta)} 
    - \underbrace{(r-1) \overline{\sigma}_{r-1}\left\{ \|\Lambda y_i + \eta \| \right\}}_{g(\Lambda,\eta)},
\end{align}
{\color{black}where $\overline{\sigma}_r$ and $\overline{\sigma}_{r-1}$ denote the averages of the $r$ and $r-1$ largest values in the set. A formal proof that $f$ and $g$ are convex functions is provided in Appendix~\ref{app:convexity:DC}.}

A key advantage of the DC reformulation in~\eqref{eq:obj:funct:min:vol} is that it allows for efficient local optimization using the DC algorithm (DCA) \cite{tao1997convex,horst1999dc}. This method iteratively constructs convex approximations of the objective function, solving a convex subproblem at each step rather than optimizing the original nonconvex function directly. This approach ensures scalability while preserving the problem's structure. We briefly outline the DCA approach.

\paragraph{Difference-of-convex algorithm (DCA).}
DCA begins with an initial point $(\Lambda_0, \eta_0) \in \mathrm{dom} \, \partial g$ and iteratively refines the solution by linearizing the concave part of the objective function. At each iteration, the function $g(\Lambda, \eta)$ is linearized at the current iterate $(\Lambda_t, \eta_t)$, yielding a convex majorization of $f(\Lambda, \eta) - g(\Lambda, \eta)$. The next iterate $(\Lambda_{t+1}, \eta_{t+1})$ is then obtained by solving the resulting convex subproblem. Mathematically, given a subgradient $(G_{\Lambda_t}, G_{\eta_t}) \in \partial g(\Lambda_t, \eta_t)$, the update rule is
\begin{align}
\label{eq:DCA}
    (\Lambda_{t+1}, \eta_{t+1}) \in \arg\min \left\{ f(\Lambda, \eta) - \mathrm{tr} \left( G_{\Lambda_t}{}^\top \Lambda \right) - (G_{\eta_t})^\top \eta : \Lambda \geqm 0, \eta \in \mathbb R^k \right\},
\end{align}
where the matrix subgradient $G_{\Lambda_t}$ and the vector subgradient $G_{\eta_t}$ are defined by
\begin{align*}
    G_{\Lambda_t} &\coloneqq \sum_{i=1}^{n} \mathds{1} \left\{ \|\Lambda_t y_i + \eta_t\| \geq \sigma_{r-1}\{\|\Lambda_t y_i + \eta_t\|\} \right\} g_i y_i^\top, \\
G_{\eta_t} &\coloneqq \sum_{i=1}^{n} \mathds{1} \left\{ \|\Lambda_t y_i + \eta_t\| \geq \sigma_{r-1}\{\|\Lambda_t y_i + \eta_t\|\} \right\} g_i,
\end{align*}
and where the vectors $g_i$ belong to the subdifferential of the norm function:
\begin{align*}
    g_i \in \partial \|\Lambda_t y_i + \eta_t\| = \left\{ g \in \mathbb{R}^k \mid \langle g, \Lambda_t y_i + \eta_t \rangle = \|\Lambda_t y_i + \eta_t\|, \; \|g\|_* \leq 1 \right\},
\end{align*}
with $\|\cdot\|_*$ being the dual norm of $\|\cdot\|$. This definition ensures that $ g_i $ acts as a supporting hyperplane at the point $ \Lambda_t y_i + \eta_t $ with respect to the given norm structure. If the norm is differentiable (which holds for $ p > 1 $ and $ x \neq 0 $), the subgradient simplifies to:
\begin{align*}
    (g_i)_j = \frac{|(\Lambda_t y_i + \eta_t)_j|^{p-1} \operatorname{sgn}((\Lambda_t y_i + \eta_t)_j)}{\|\Lambda_t y_i + \eta_t\|_p^{p-1}}, \quad \text{for all } j = 1, \dots, k, \; \text{and for } p > 1.
\end{align*}
In the special case of $ p=2 $, this reduces to $g_i = {\Lambda_t y_i + \eta_t}/{\|\Lambda_t y_i + \eta_t\|_2}$. {\color{black}In practice, problem~\eqref{eq:DCA} can be solved efficiently using \texttt{cvxpy} \cite{diamond2016cvxpy} with the Mosek solver \cite{mosek}}.

A key property of the DCA is that the sequence of objective values $ f(\Lambda_t, \eta_t) - g(\Lambda_t, \eta_t) $ is \emph{nonincreasing and convergent} (see, e.g., \cite[Lemma~1]{niu2022convergence}). However, since problem~\eqref{eq:min:vol:exact} is nonconvex, DCA does not guarantee global optimality. The algorithm may converge to a local minimum, meaning that the resulting prediction set may not have the smallest possible volume.

\begin{figure}[t]
    \center
    \subfigure{\includegraphics[width=0.48\columnwidth]{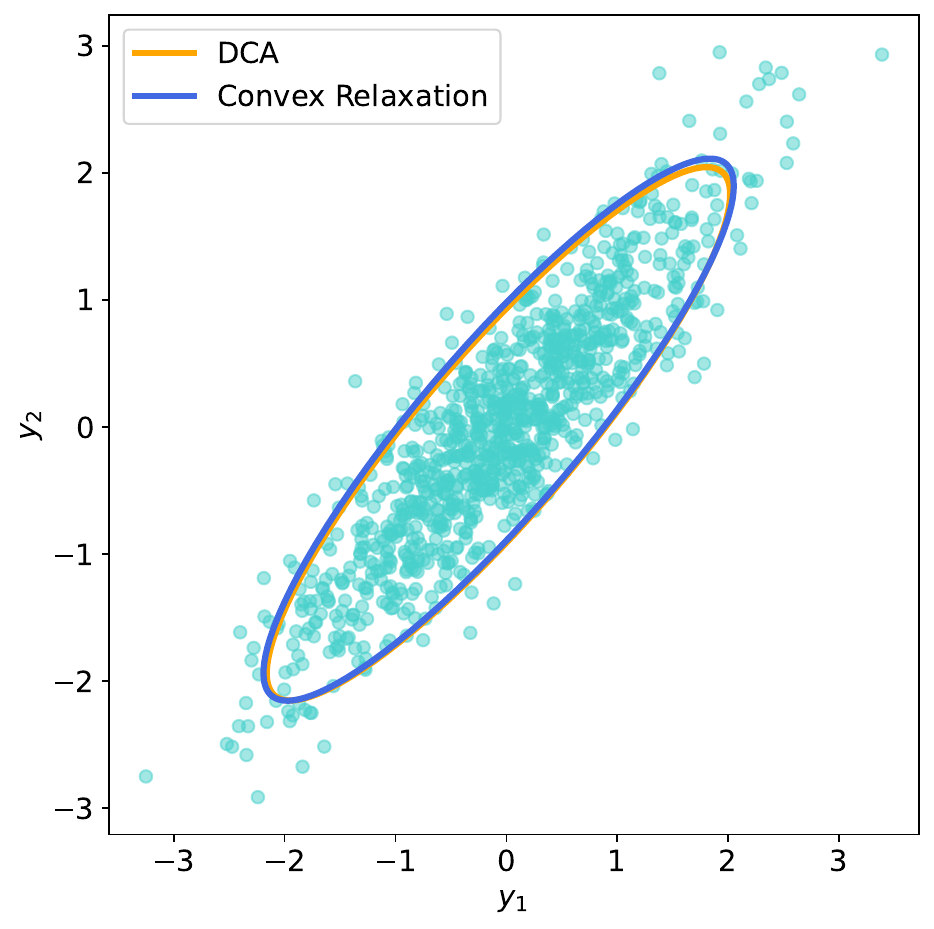}} ~
    \subfigure{\includegraphics[width=0.48\columnwidth]{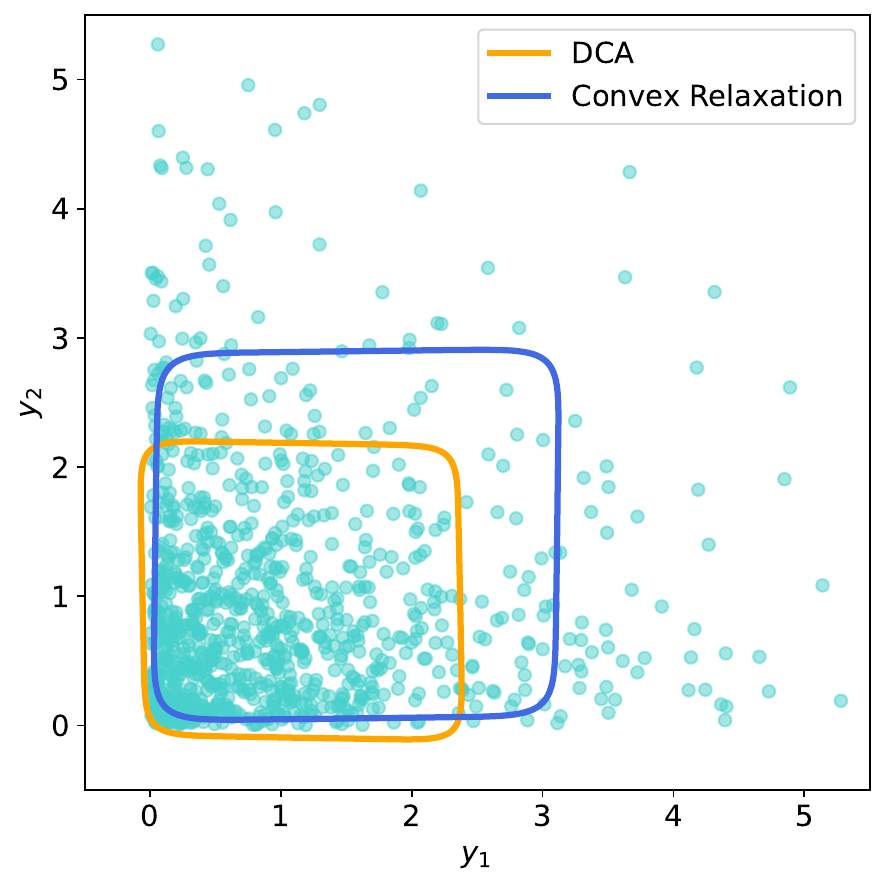}} 
    \vspace{-5mm}
    \caption{Comparison of the MVCSs obtained via the convex relaxation problem~\eqref{eq:min:vol:convex} and the DCA method~\eqref{eq:DCA}. The solid lines represent the boundaries of $\mathbb{B}(p, M, \mu)$, with dots corresponding to i.i.d.\ samples from the underlying distribution. (Left) Gaussian distribution: $Y\sim \mathcal{N}\left(0, \left({\begin{smallmatrix} 1 & 0.9 \\ 0.9 & 1 \end{smallmatrix}}\right) \right)$ with coverage level $0.90$ and $p = 2$. Both methods produce similar sets, aligning well with the empirical covariance structure. (Right) Exponential distribution: $Y\sim \mathcal{E}(1)^{\otimes 2}$ with coverage level $0.80$ and $p = 10$. While the convex relaxation results in a more elongated set influenced by extreme values, the DCA solution better captures the dense central region of the distribution.}
    \label{figure:DC:CR}
\end{figure}

\paragraph{Convex relaxation.}
Instead of solving the nonconvex DC formulation via the DCA algorithm~\eqref{eq:DCA}, an alternative approach is to consider a convex relaxation. This is obtained by dropping the term $(r-1) \overline{\sigma}_{r-1}\left\{ \|\Lambda y_i + \eta \| \right\}$ from the objective function in \eqref{eq:obj:funct:min:vol}, leading to the convex optimization problem
\begin{align} 
\label{eq:min:vol:convex} 
\begin{split}
    \min \quad & - \log \det (\Lambda) + \sum_{i=1}^n \max\{ \|\Lambda y_i + \eta\| - \nu, 0\} + r \nu \\ 
    \mathrm{s.t.} \quad & \Lambda \geqm 0, \; \eta \in \mathbb{R}^k, \; \nu \in \mathbb R,
\end{split} 
\end{align}
where $\nu$ is an auxiliary variable introduced by reformulating $\overline{\sigma}_r\left\{ \|\Lambda y_i + \eta \| \right\}$ as a minimization problem (see Appendix~\ref{app:convexity:DC} for details). This formulation can be efficiently solved using standard off-the-shelf solvers. However, since the term $\overline{\sigma}_{r-1}$ played a crucial role in balancing coverage constraints with volume minimization, the relaxed formulation may lead to larger prediction sets than the exact DC formulation. Nevertheless, in scenarios where computational efficiency is prioritized, this convex relaxation provides a scalable alternative.

\begin{remark}[MVCS for the DCA and convex relaxation problem]
Let $(\Lambda^\star,\eta^\star)$ denote the final iterate of the DCA method~\eqref{eq:DCA} or the global solution of the convex relaxation~\eqref{eq:min:vol:convex}. The corresponding MVCS set, which contains $n - r + 1$ points, is given by $\{ y \in \mathbb R^k  \mid  \| \Lambda^\star y + \eta^\star \| \leq \sigma_r\left\{ \|\Lambda^\star y_i + \eta^\star\| \right\} \}$. We illustrate in Figure~\ref{figure:DC:CR} the solutions obtained by both methods for different distributions. In the Gaussian case (left), the solutions are nearly identical, with both approaches aligning with the empirical covariance structure. However, for an asymmetric distribution (right), the convex relaxation method is more influenced by the tails, stretching the prediction set outward, whereas the DCA solution remains concentrated around the high-density region near the origin. These observations suggest that in elliptical distributions, the convex relaxation method performs well, but for non-elliptical distributions, the DCA method provides tighter, more representative sets.
\end{remark}

{\color{black}\begin{remark}[Connection to pinball loss]
A key insight emerges when taking the derivative of the objective function in~\eqref{eq:min:vol:convex} with respect to $\nu$: at optimality, $\nu^\star$ is the $(n - r + 1)$-th order statistic of the nonconformity scores, enforcing the required coverage constraint. This mechanism closely parallels the structure of quantile regression \cite{koenker2005quantile}, where the pinball loss implicitly selects a quantile level by balancing over- and under-estimation errors. A similar argument appears in \cite[Lemma 2.2.1]{roth2022uncertain}, which establishes that minimizing the pinball loss aligns with quantile estimation through a first-order optimality condition.
\end{remark}}

%------------------------------------------------------------------------

\subsection{Learning $p$-norm prediction sets}
\label{subsec:p:norms}

Thus far, we have considered the problem of finding the minimum-volume norm-based set while assuming that the norm structure is fixed. We now extend this approach by allowing the norm itself to be optimized. To accomplish this, we restrict our attention to the family of $p$-norms, parameterized by $p \in (0, \infty)$. This introduces an additional degree of flexibility, enabling the prediction set to better conform to the geometry of the data. Recalling \eqref{eq:min:vol}, the MVCS problem becomes
\begin{align}
\label{eq:min:vol:problem:p}
\begin{split}
    \min \quad & \textrm{Vol}(\mathbb{B}(p, M, \mu)) \\
    \mathrm{s.t.} \quad & M \geqm 0, \; \mu \in \mathbb{R}^k, \; p > 0, \\
    & \mathrm{Card} \left\{ i \in [n] \mid \|M (y_i - \mu)\|_p \leq 1 \right\} \geq n - r + 1.
\end{split}
\end{align}
Here, the volume of the $p$-norm-based set is still given by $\textrm{Vol}(\mathbb{B}(p, M, \mu)) = \lambda(B_{\|\cdot\|_p}(1)) \cdot \det (M)^{-1}$, where $\lambda(B_{\|\cdot\|_p}(1))$ is the Lebesgue measure of the unit $p$-norm ball. Its explicit formula is given by \cite{dirichlet1839sue},
\begin{align}
\label{eq:Lebesgue:p:norm}
    \lambda(B_{\|\cdot\|_p}(1)) = \frac{2^k\, \Gamma(1 + 1/p)^k}{\Gamma(1 + k/p)},
\end{align}
where $\Gamma(\cdot)$ denotes the gamma function \cite{artin2015gamma}. For brevity, we will retain the notation $\lambda(B_{\|\cdot\|_p}(1))$ throughout the remainder of the paper. Building on the reasoning in Proposition~\ref{prop:min:vol:exact}, we derive the following corollary, which provides a more structured and tractable reformulation of the MVCS problem~\eqref{eq:min:vol:problem:p}.

\begin{corollary}[Single-norm MVCS]
\label{cor:min:vol:exact:p}
The MVCS problem~\eqref{eq:min:vol:problem:p} has the same optimal value as the following optimization problem:
\begin{align}
\label{eq:min:vol:exact:p}
\begin{split}
    \min \quad & - \log \det (\Lambda) + k \log \sigma_r\left\{ \|\Lambda (y_i + \mu) \|_p \right\} + \log \lambda(B_{\|\cdot \|_p}(1)) \\
    \mathrm{s.t.} \quad & \Lambda \geqm 0,\; \mu \in \mathbb{R}^k,\; p > 0,
\end{split}
\end{align}
and the optimal $M^\star$ in \eqref{eq:min:vol:problem:p} can be recovered from the optimal solution $\Lambda^\star$ of \eqref{eq:min:vol:exact:p} as
\begin{align*}
    M^\star = \sigma_r\left\{ \|\Lambda^\star (y_i + \mu^\star)\|_{p^\star} \right\}^{-1} \Lambda^\star.
\end{align*}
\end{corollary}

Allowing $p$ to be learned from the data provides additional flexibility in shaping the prediction set. Unlike fixed-norm approaches that impose a predetermined geometry, optimizing $p$ enables the set to adapt to the distributional characteristics of the data. When $p$ is small, the prediction set tends to form elongated and anisotropic shapes, reflecting dominant directions of variability in the data. In contrast, larger values of $p$ yield more uniform contours, with $p = 2$ corresponding to ellipsoidal structures and $p \to \infty$ resulting in axis-aligned hyperrectangles. By jointly optimizing $\Lambda$, $\mu$, and $p$, we obtain a method that automatically selects the most appropriate norm, balancing volume minimization and coverage constraints. The optimal solution $(\Lambda^\star, \mu^\star, p^\star)$ obtained from problem~\eqref{eq:min:vol:exact:p} defines the minimum-volume covering set. Specifically, the learned $p$-norm-based prediction region is given by $\left\{ y \in \mathbb{R}^k  \mid  \| \Lambda^\star (y + \mu^\star) \|_{p^\star} \leq \sigma_r\left\{ \|\Lambda^\star (y_i + \mu^\star)\|_{p^\star} \right\} \right\}$.

\begin{figure}[t]
    \center
    \subfigure{\includegraphics[width=0.48\columnwidth]{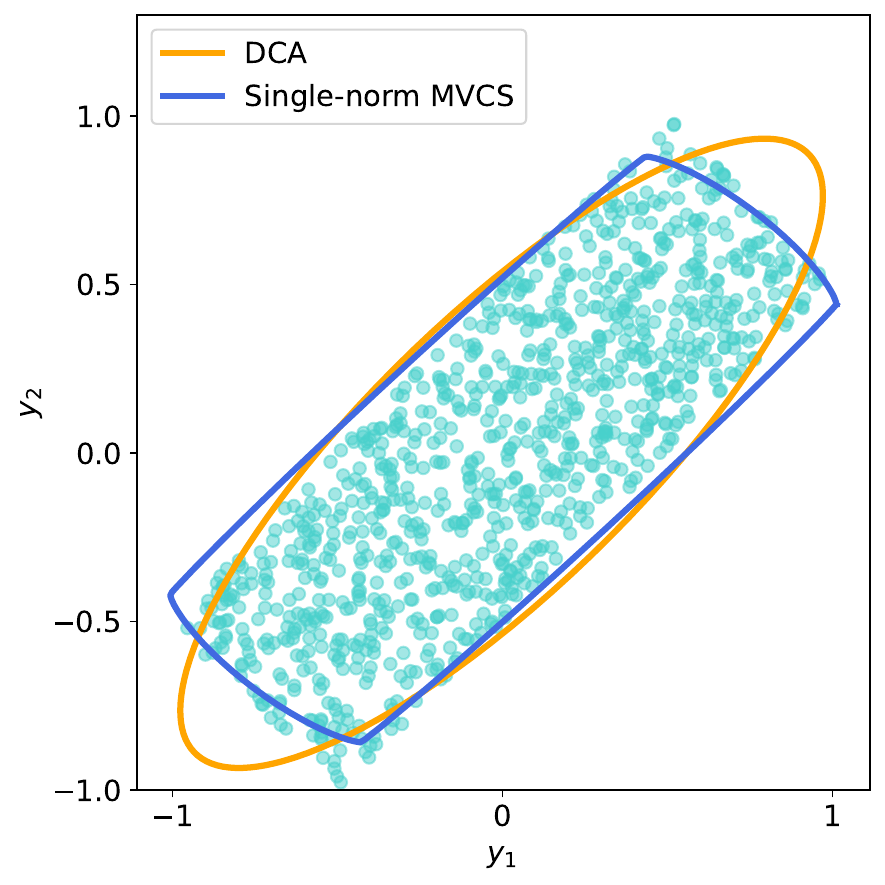}} ~
    \subfigure{\includegraphics[width=0.48\columnwidth]{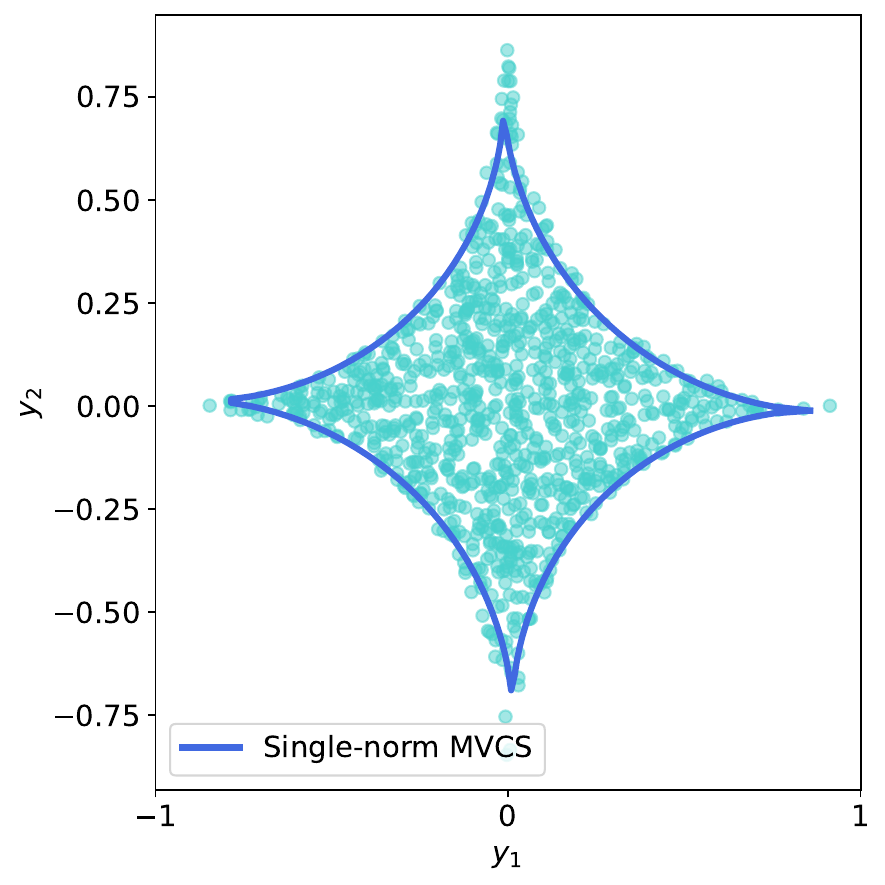}} 
    \vspace{-5mm}
    \caption{(Left) Comparison of MVCS sets obtained via the first-order optimization strategy described in Section~\ref{subsec:p:norms} for solving problem~\eqref{eq:min:vol:exact:p} and the DCA approach with $p=2$. Both methods achieve the target coverage level of $0.95$, but the learned MVCS optimizes $p = 1.11$, leading to a more adaptive shape. (Right) MVCS obtained via~\eqref{eq:min:vol:exact:p} for a dataset exhibiting strong anisotropy. The optimization procedure selects $p = 0.58$, resulting in an elongated set that better captures the structure of the data while achieving the target coverage level of $0.90$.}
    \label{figure:learn:p}
\end{figure}

\paragraph{First-order optimization strategy.}
While the optimization problem~\eqref{eq:min:vol:exact:p} is already nonconvex due to the function $\sigma_r$, the additional optimization over $p$ further complicates the landscape, breaking the difference-of-convex (DC) structure that allowed for efficient iterative optimization in the fixed-norm case. Consequently, direct application of the DCA is no longer possible, and solving for $p$ requires alternative strategies. {\color{black}We thus turn to first-order optimization methods, as explained next.

\begin{remark}[Unconstrained single-norm MVCS]
The primary challenge in designing such an algorithm is handling the constraint $\Lambda \geqm 0$. To eliminate this constraint, we parameterize $\Lambda$ as $\Lambda := A A^\top$, where $A \in \mathbb{R}^{k \times k}$. This parameterization ensures positive semidefiniteness without imposing additional assumptions on $A$. Additionally, to allow for unconstrained optimization over $p$, we reparameterize it as $|p|$, enabling updates without restrictions on its sign. With these transformations, the optimization problem reduces to minimizing the following \emph{unconstrained} loss function,
\begin{align*}
    - \log \det (A A^\top) + k \log \sigma_r\left\{ \|A A^\top (y_i + \mu) \|_{|p|} \right\} + \log \lambda(B_{\|\cdot \|_{|p|}}(1)),
\end{align*} 
over $A \in \mathbb{R}^{k \times k}$, $\mu \in \mathbb{R}^k$, and $p \in \mathbb{R}$. For this objective, standard first-order optimization algorithms, such as gradient descent, can be applied. For our experiments, we update gradients using the Adam optimizer \cite{kingma2014adam} in \texttt{PyTorch}, with a learning rate scheduler (see Appendix~\ref{app:hyperparameters}). After optimization, the final transformation is recovered as $\Lambda^* = A^* A^{*\top}$.
\end{remark}}

Despite the challenges induced by the nonconvexity  \cite{lee2016gradient}, we observe empirically that the joint optimization over {\color{black}$(A, \mu, p)$} with nonconvex gradient descent performs remarkably well in practice. In particular, our numerical experiments demonstrate that the learned values of $p$ often yield well-calibrated prediction sets that align with the intrinsic geometric properties of the data. We illustrate this in Figure~\ref{figure:learn:p}, where the left panel compares the MVCS obtained via problem~\eqref{eq:min:vol:exact:p} to the DCA approach with fixed $p=2$. While both maintain the prescribed coverage level, optimizing $p$ yields a more flexible shape that better fits the data distribution. The right panel further demonstrates this effect in an anisotropic setting, where the learned $p=0.58$ results in an elongated prediction set, adapting to the underlying structure.

Further illustrating the flexibility of our approach, Figure~\ref{figure:learn:3D} presents two three-dimensional examples where the learned values of $p$ adapt to the underlying data distribution. The left panel shows an elongated prediction set with $p=0.56$, learned from data uniformly distributed in an $\ell_{0.5}$-ball, capturing its directional spread. In contrast, the right panel corresponds to an exponential distribution, where the learned $p=8.19$ results in an axis-aligned prediction region, similar to the patterns observed in the right panel of Figure~\ref{figure:DC:CR}.

\begin{figure}[t]
    \center
    \subfigure{\includegraphics[width=0.48\columnwidth]{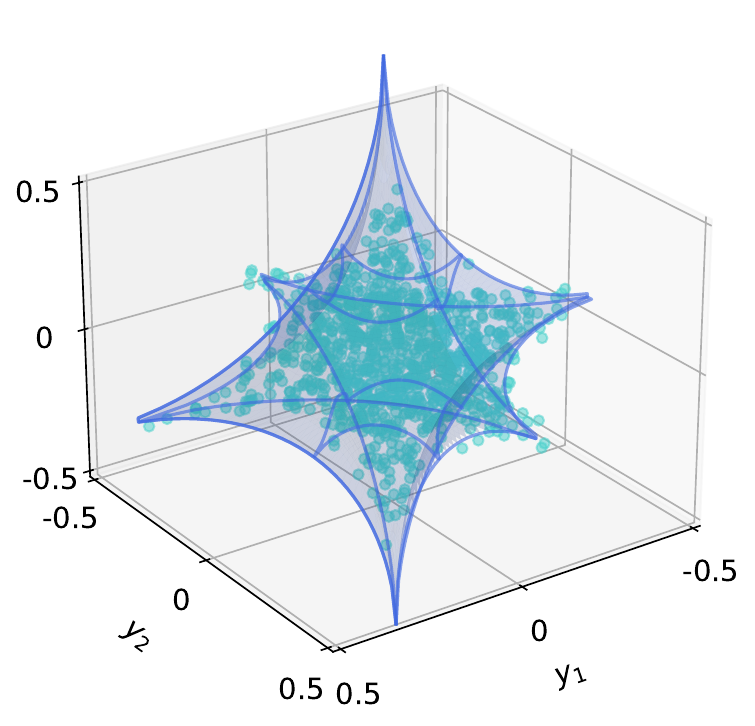}} ~
    \subfigure{\includegraphics[width=0.48\columnwidth]{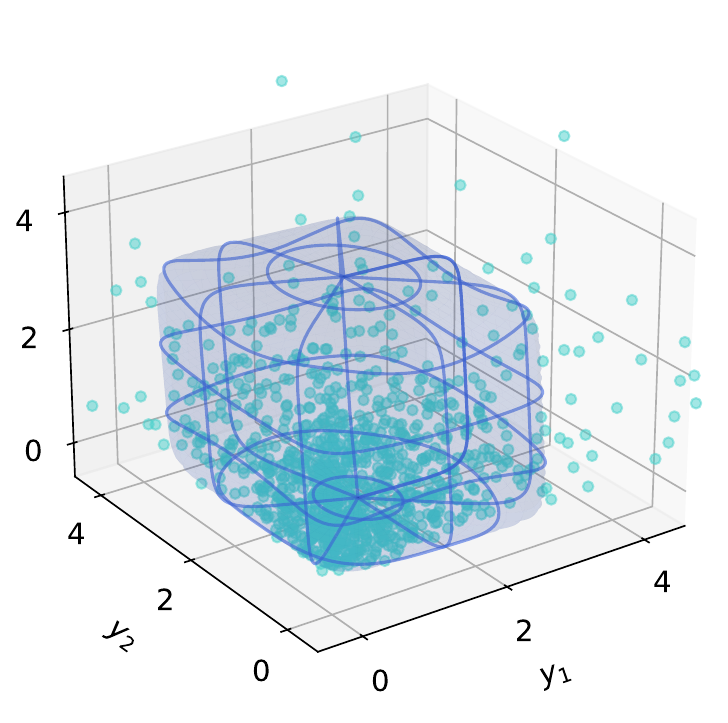}} 
\vspace{-5mm}
\caption{Visualization of MVCSs obtained via problem~\eqref{eq:min:vol:exact:p} in three dimensions. (Left) Prediction set learned for a uniform distribution over the 0.5-radius ball, with optimized $p=0.56$ adapting to the anisotropic structure of the data. (Right) MVCS learned for an exponential distribution, where $p=8.19$ results in an axis-aligned structure that reflects the concentration of the data.}
    \label{figure:learn:3D}
\end{figure}

\paragraph{Learning multi-norm prediction sets.}
Thus far, we have optimized prediction sets using a single global norm. While this provides flexibility, it fails to adapt to local variations in uncertainty. In many applications, uncertainty behaves asymmetrically—certain directions require tighter coverage, while others necessitate broader regions. To address this, we extend our framework to allow for region-dependent norms by partitioning the space into $m$ disjoint regions $\{\mathcal A_j\}_{j=1}^{m}$, each with its own norm~$p_j$. A natural way to define these partitions is through axis-aligned decomposition, where the space is split along coordinate axes. This approach ensures computational tractability and facilitates structured volume computations, allowing for at most $m = 2^k$ regions in dimension $k$. For example, when $k=2$, the space can be split into four quadrants, and when $k=3$, it can be divided into eight octants.

To construct these prediction sets, we introduce a global rotation matrix $R$ that aligns the space, followed by region-specific \emph{diagonal} scaling matrices $D_j$, for $j \in [m]$. The role of $R$ is to standardize the representation of data by aligning principal directions of variation with coordinate axes. Within each partition, the diagonal matrix $D_j$ then scales the axes independently, capturing anisotropic structure while preserving computational efficiency. Given a center $\mu$, we define the pseudo-distance 
\begin{align*}
    d(y, \mu; {R, \{D_j\}_{j=1}^m}) \coloneqq \sum_{j=1}^{m} \mathds{1}\{R(y - \mu) \in \mathcal A_j \} \left\|R D_j (y - \mu) \right\|_{p_j},
\end{align*}  
which induces a prediction set that takes the form
\begin{align*}
    \mathbb{M}\left( \{p_j\}_{j=1}^m, R, \{D_j\}_{j=1}^m, \mu \right) \coloneqq \left\{ y \in \mathbb{R}^k \mid d(y, \mu; {R, \{D_j\}_{j=1}^m}) \leq 1 \right\}.
\end{align*}
Using this decomposition, the multi-norm MVCS problem is formulated as
\begin{align}
\label{eq:min:vol:multi-p}
\begin{split}
\min \quad & \mathrm{Vol}\left(\mathbb{M}\left( \{p_j\}_{j=1}^m, R, \{D_j\}_{j=1}^m, \mu \right)\right)\\
\mathrm{s.t.} \quad & R \in \mathrm{SO}(k)
, \; \mu \in \mathbb{R}^k,\; D_j \in \mathrm{Diag}(k) \\
& D_j \geqm 0, \; p_j > 0, \quad \forall\, j \in [m] \\
& \mathrm{Card} \left\{ i \in [n] \mid d(y_i, \mu; {R, \{D_j\}_{j=1}^m}) \leq 1 \right\} \geq n - r + 1.
\end{split}
\end{align}
Here, the total volume of the multi-norm prediction set is given by the average of the region-specific volumes,
\begin{align}
\label{eq:vol:multi-p}
\mathrm{Vol}\left(\mathbb{M}\left( \{p_j\}_{j=1}^m, R, \{D_j\}_{j=1}^m, \mu \right)\right) = \frac{1}{m} \sum_{j=1}^{m} \mathrm{Vol}\left(\mathbb{B}(p_j,D_j,\mu)\right) = \frac{1}{m} \sum_{j=1}^{m}{\lambda(B_{\|\cdot \|_{p_j}}(1))}{\det(D_j)}^{-1},
\end{align}
where $\lambda(B_{\|\cdot \|_{p_j}}(1))$ represents Lebesgue measure of the unit $p_j$-norm ball, given explicitly in \eqref{eq:Lebesgue:p:norm}. The global rotation matrix $R$ does not affect the volume since it is an orthogonal transformation with $\det(R) = 1$, leaving the determinant term unchanged. Moreover, the uniform weighting of $1/m$ arises because the matrices $D_j$ are diagonal, so the total volume is taken as the average contribution across all regions. The following proposition provides a reformulation of \eqref{eq:min:vol:multi-p} that is amenable to iterative optimization algorithms.

\begin{proposition}[Multi-norm MVCS]
\label{prop:min:vol:multi}
The multi-norm MVCS problem~\eqref{eq:min:vol:multi-p} is equivalent to the following optimization problem
\begin{align}
\label{eq:min:vol:exact:multi-p}
\begin{split}
\min \quad & k \log \left(\sigma_r\left\{ d(y_i, \mu; {R, \{D_j\}_{j=1}^m}) \right\}\right) + \log \left(\sum_{j=1}^m {\lambda(B_{\|\cdot \|_{p_j}}(1))}{\det(D_j)}^{-1}\right) \\
\mathrm{s.t.} \quad & R \in \mathrm{SO}(k)
, \; \mu \in \mathbb{R}^k,\; D_j \in \mathrm{Diag}(k) \\
& D_j \geqm 0, \; p_j > 0, \quad \forall\, j \in [m].
\end{split}
\end{align}
\end{proposition}
\begin{proof}
Using the change of coordinates $D_j \to D_j/\nu$, for $\nu \geq 0$, problem~\eqref{eq:min:vol:multi-p} is equivalent to
\begin{align}
\label{eq:min:vol:exact:multi-p:temp:1}
\begin{split}
\min \quad & k \log \nu + \log \left(\sum_{j=1}^m {\lambda(B_{\|\cdot \|_{p_j}}(1))}{\det(D_j)}^{-1}\right) \\
\mathrm{s.t.} \quad & R \in \mathrm{SO}(k)
, \; \mu \in \mathbb{R}^k,\; D_j \in \mathrm{Diag}(k),\; \nu \geq 0 \\
& D_j \geqm 0, \; p_j > 0, \quad \forall\, j \in [m] \\
& \mathrm{Card} \left\{ i \in [n] \mid d(y_i, \mu; {R, \{D_j\}_{j=1}^m}) \leq \nu \right\} \geq n - r + 1,
\end{split}
\end{align}
where we have used the expression of the volume provided in \eqref{eq:vol:multi-p}. Now, due the cardinality constraint, for given $R, \mu, \{D_j\}_{j=1}^m$, and $\{p_j\}_{j=1}^m$, the optimal $\nu^\star(R, \mu, \{D_j\}_{j=1}^m, \{p_j\}_{j=1}^m)$ is equal to the $(n-r+1)$-th smallest value in the set $\left\{d(y_i,\mu; {R, \{D_j\}_{j=1}^m})\right\}_{i=1}^n$, or, equivalently, the $r$-th largest value, i.e., $\nu^\star = \sigma_r\left\{ d(y_i,\mu; {R, \{D_j\}_{j=1}^m}) \right\}$. Plugging this into \eqref{eq:min:vol:exact:multi-p:temp:1}, we obtain problem \eqref{eq:min:vol:exact:multi-p} in the statement of the proposition.
\end{proof}

\textcolor{black}{
\begin{remark}[Unconstrained multi-norm MVCS]
Similarly to the single-norm approach, we can express the objective function in an unconstrained form as:
\begin{align*}
    k \log \left(\sigma_r\left\{ d(y_i, \mu; {R, \{|D_j|\}_{j=1}^m}) \right\}\right) + \log \left(\sum_{j=1}^m {\lambda(B_{\|\cdot \|_{|p_j|}}(1))}{\det(|D_j| )}^{-1}\right).
\end{align*}
To enable unconstrained optimization over the rotation matrix $R$, we introduce a parameterization based on the QR decomposition. Specifically, we optimize over a full matrix $Q \in \mathbb{R}^{k \times k}$ and extract the rotation matrix via the decomposition $Q = R_1 R_2$, where $R_1$ is an orthogonal matrix and $R_2$ is upper triangular. Since $R_1$ may have determinant $-1$, we ensure that $R$ remains a proper rotation matrix in $\mathrm{SO}(k)$ by setting 
$R = R_1$ if $\det R_1 = 1$, and $R = I_{-} R_1$ if $\det R_1 = -1$, where $I_{-}$ is a diagonal matrix with all ones except for a $-1$ in the first diagonal entry.
\end{remark}}

\begin{figure}[t]
    \center
    \subfigure{\includegraphics[width=0.48\columnwidth]{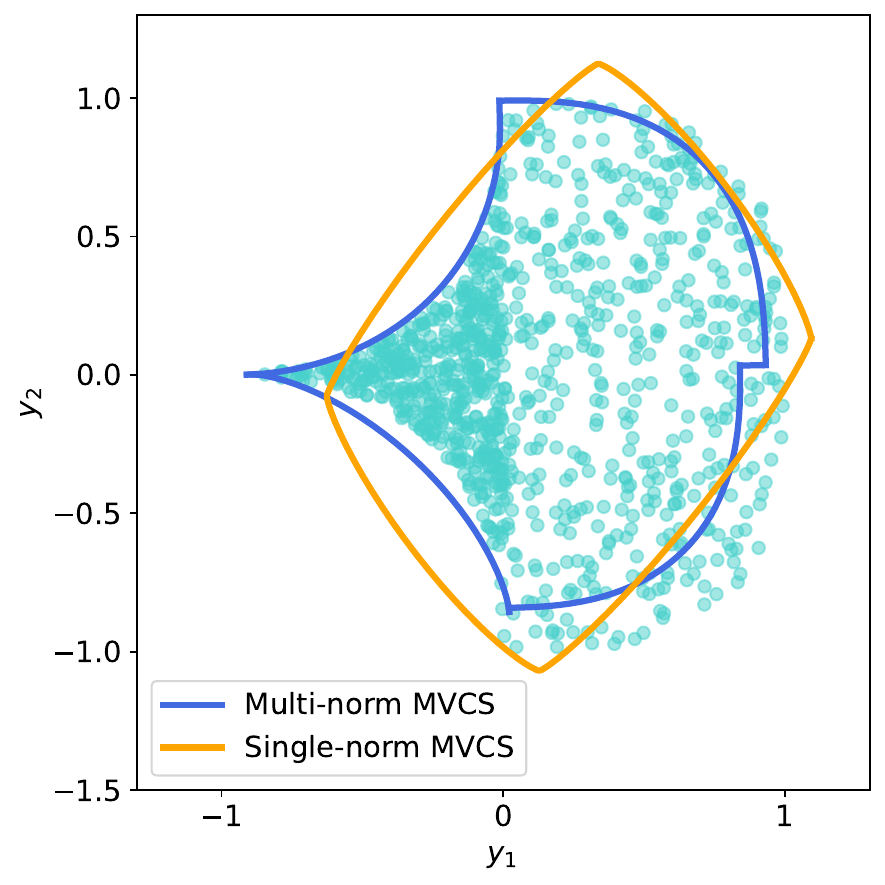}} ~
    \subfigure{\includegraphics[width=0.48\columnwidth]{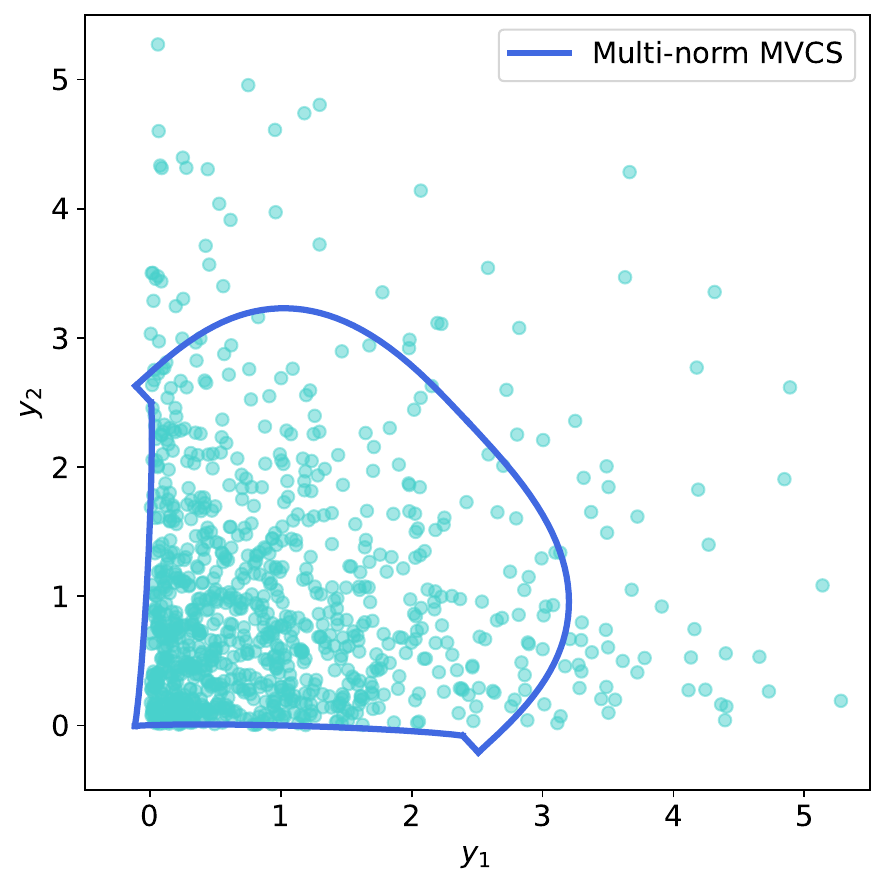}} 
\vspace{-5mm}
\caption{(Left) Comparison of MVCS sets obtained using a single learned norm versus a multi-norm formulation. The single-norm MVCS is learned with $p = 1.21$, while the multi-norm MVCS optimizes over four different norms, yielding region-specific values of $p = 0.68, 0.56, 2.52, 2.59$. (Right) A separate example of multi-norm MVCS with two learned norms, $p = 0.96, 2.90$. In both cases, the multi-norm approach leads to a more adaptive shape that captures the anisotropic structure of the data while maintaining the prescribed coverage level of $0.90$.}
    \label{figure:learn:multi:p}
\end{figure}

Problem~\eqref{eq:min:vol:exact:multi-p} provides a principled framework for learning adaptive multi-norm prediction sets, but it is nonconvex. Nonetheless, its structured formulation again allows for effective optimization using iterative approaches to obtain locally optimal solutions. Empirical results demonstrate that the learned multi-norm prediction sets significantly enhance flexibility and adaptivity, effectively capturing complex uncertainty structures across different regions of space. We illustrate this in Figure~\ref{figure:learn:multi:p}, where we compare MVCS sets obtained via the iterative procedure described in Section~\ref{subsec:training:procedure} for solving problem~\eqref{eq:min:vol:exact:multi-p}. The left panel contrasts the single-norm MVCS, learned with $p = 1.21$, against a multi-norm formulation that assigns distinct norms to different regions of the space. The multi-norm MVCS successfully adapts to the anisotropic data structure by varying $p$ across regions, yielding a more expressive and data-conforming prediction set. In the right panel, we further observe that learning multiple norms refines the predictive set shape, adjusting its geometry based on local uncertainty.

%------------------------------------------------------------------------

\subsection{Application to multivariate regression}
\label{subsec:application:regression}

In many practical scenarios, we aim to construct prediction sets for a multivariate response variable $Y \in \mathbb{R}^k$ given a covariate vector $X \in \mathbb{R}^d$. The MVCS framework naturally extends to this setting by modeling uncertainty through residuals relative to a predictive model. Let $f_\theta: \mathbb{R}^d \to \mathbb{R}^k$ be a parametric model with parameters $\theta$. If $f_\theta$ is fixed (i.e., fit on a separate dataset), we can apply the MVCS methodology to the residuals $\{y_i - f_\theta(x_i)\}_{i=1}^{n}$ to construct the prediction region
\begin{align*}
    C(x) \coloneqq f_\theta(x) + \mathbb{B}(p, M, \mu),
\end{align*}
where $\mathbb{B}(p, M, \mu)$ is learned using Corollary~\ref{cor:min:vol:exact:p}. The role of the correction term $\mu$ is to optimally position the covering set within the residual space. Unlike a shift meant to correct predictive bias in $f_\theta(x)$, $\mu$ aligns the set to minimize its volume while still maintaining coverage. This approach yields feature-dependent prediction regions by centering the set at $f_\theta(x) + \mu$. However, the estimated set remains \emph{globally constrained}, as the same structure is applied uniformly across all values of $x$. 
In Section~\ref{sec:multivariate:regression}, we introduce \emph{local adaptivity}, allowing the prediction set to adjust to variations in $x$ and better capture heteroskedastic uncertainty.

\paragraph{Optimizing over $\theta$.}
Rather than fixing $f_\theta$, a natural extension is to \textit{jointly optimize} over $\theta$ together with the MVCS parameters. This formulation seeks the predictive model that minimizes the volume of the covering set, effectively balancing predictive accuracy and uncertainty quantification. By directly learning $\theta$ to align with the minimum-volume criterion, the model adapts to the data in a way that tightens the overall uncertainty region. 

A key consequence of this joint optimization is that the correction term $\mu$ becomes unnecessary. When $f_\theta$ is fixed, $\mu$ is used to optimally reposition the prediction set. However, when optimizing $f_\theta$, the model itself can absorb any systematic shift. As a result, the optimization problem becomes
\begin{align}
\label{eq:min:vol:regression-theta}
\begin{split}
    \min \quad & - \log \det (\Lambda) + k \log \sigma_r\left\{ \|\Lambda (y_i - f_\theta(x_i)) \|_p \right\} + \log \lambda(B_{\|\cdot\|_p}(1)) \\
    \mathrm{s.t.} \quad & \Lambda \geqm 0,\; p > 0,\; \theta \in \Theta,
\end{split}
\end{align}
where $\Theta$ represents the space of permissible model parameters. This formulation has several advantages:
\begin{itemize}
    \item \emph{Unified learning.} The prediction function $f_\theta$ is no longer a separate entity; it is learned jointly with the uncertainty quantification model.
    \item \emph{Implicit set alignment.} Any shift that would have been captured by $\mu$ is now directly absorbed by the optimization over $f_\theta$.
    \item \emph{Tighter prediction sets.} The predictive model is adjusted not just for accuracy, but to yield the minimum-volume uncertainty set.
\end{itemize}

The joint optimization introduces additional computational complexity, but empirical results indicate that the method performs well in practice, effectively balancing predictive accuracy and uncertainty quantification while adapting to the structure of the data. We illustrate this in Figure~\ref{figure:learn:theta}, which presents the MVCSs constructed for different feature values $X_i$ in a setting where the noise distribution varies with $X$ (see Section~\ref{sec:numerical:validation} for details). The learned norms adapt to the residual structure, leading to different prediction set geometries. In the left panel, where the residuals are Gaussian, the optimization selects $p = 1.93$, while in the right panel, corresponding to an exponential residual distribution, the learned $p=2.24$. Notably, the same prediction set is applied uniformly across all $X$, as the transformation $M$ is learned globally rather than varying with $X$. In Section~\ref{sec:multivariate:regression}, we extend this framework to introduce local adaptivity, allowing the prediction sets to adjust to different feature values (see Figure~\ref{figure:learn:adaptive} for a direct comparison).

\begin{figure}[t]
    \center
    \subfigure{\includegraphics[width=0.48\columnwidth]{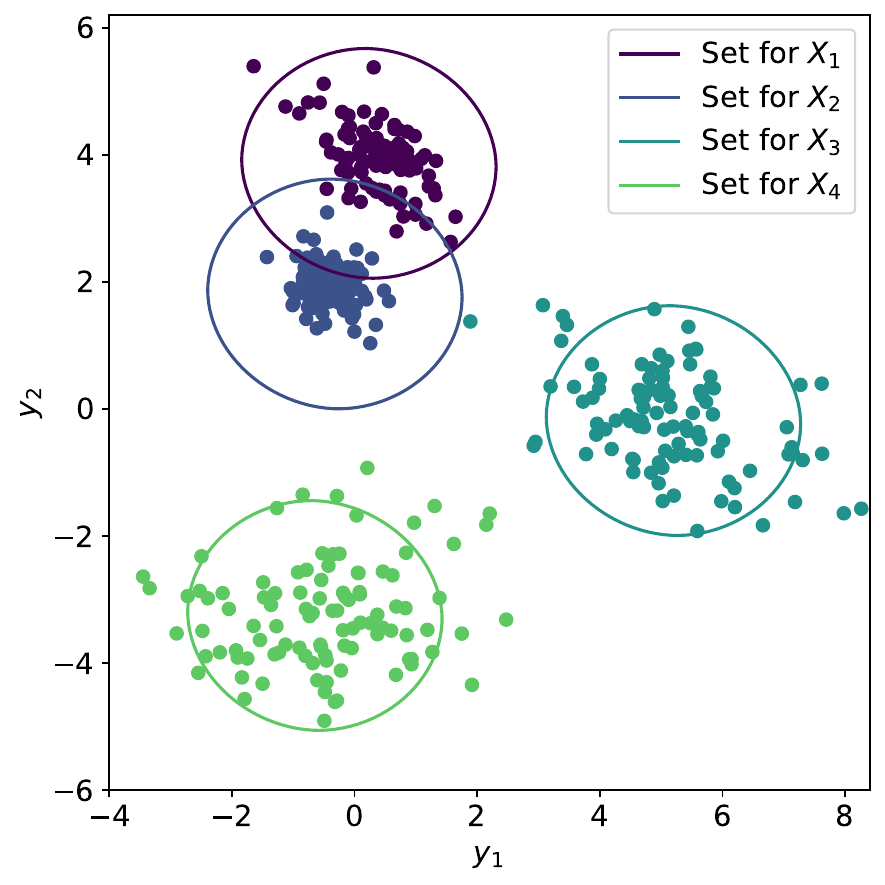}} ~
    \subfigure{\includegraphics[width=0.48\columnwidth]{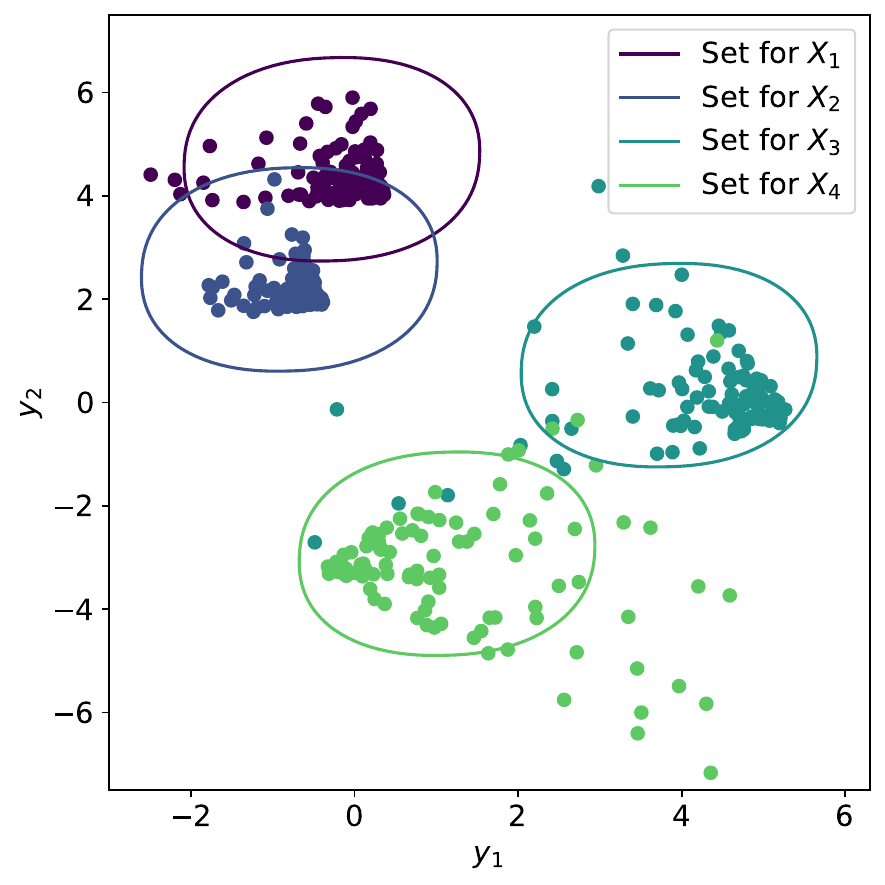}} 
    \vspace{-5mm}
    \caption{MVCSs obtained via the iterative method described in Section~\ref{subsec:training:procedure} for solving problem~\eqref{eq:min:vol:regression-theta}. Data is generated from $Y = f(X) + \text{noise}(X)$, where the noise is function of $X$ (see Appendix~\ref{app:synthetic:datageneration} for details). The models are trained on $10^4$ i.i.d.\ samples $(X_i, Y_i)$. Each color represents samples of $Y \mid X_i$ for different test points $X_i$, with the associated shape indicating the learned prediction set centered at $f_\theta(X_i)$. (Left) Gaussian residuals: learned $p = 1.93$. (Right) Exponential residuals: learned $p = 2.24$.}
    \label{figure:learn:theta}
\end{figure}

%------------------------------------------------------------------------
%------------------------------------------------------------------------
%------------------------------------------------------------------------

\section{Local Adaptivity in Multivariate Regression}
\label{sec:multivariate:regression}

In the previous section, we established a general framework for constructing minimum-volume covering sets and applied it to regression by modeling uncertainty through residuals. While this approach enables feature-dependent prediction by centering the region at $f_\theta(x) + \mu$, it assumes a globally shared uncertainty structure, applying the same prediction set across all feature values. However, residual distributions often exhibit heteroskedasticity, meaning that uncertainty varies significantly with $x$. A single global region may be overly conservative in some areas while too narrow in others. To address this, we now introduce \emph{local adaptivity}, where the prediction set varies with $x$ by learning an covariate-dependent transformation matrix $M(x)$ that adjusts the shape and size of the prediction set. Additionally, we jointly optimize the predictive model $f_\theta$ to minimize uncertainty while maintaining coverage.

Given a dataset $\{(x_i, y_i)\}_{i=1}^{n}$, we define the residuals $\{y_i - f_\theta(x_i)\}_{i=1}^n$, which quantify the discrepancy between predictions and observed responses. To account for feature-dependent variations in uncertainty, we introduce a local transformation $M(\cdot)$ and learn it from the residuals. The resulting prediction region is given by
\begin{align*}
    C(x) \coloneqq f_\theta(x) + \mathbb{B}(p, M(x), 0) = \mathbb{B}(p, M(x), f_\theta(x)),
\end{align*}
where $\mathbb{B}(p, M(x), 0)$ is a locally adaptive minimum-volume norm-based set, learned from the residuals to ensure the required coverage. As discussed in Section~\ref{subsec:application:regression}, centering the prediction set at zero is justified when additionally optimizing over $\theta$ since any systematic shift can be absorbed within $f_\theta(x)$. 

To formalize the objective of learning locally adaptive prediction sets, we consider the following population-level optimization problem:
\begin{align}
\label{eq:population:min:vol}
\begin{split}
    \min \quad & \mathbb{E} \left[ \text{Vol} (C(X)) \right] \\
    \mathrm{s.t.} \quad & \mathrm{Prob} \left\{ Y \in C(X) \right\} \geq 1 - \alpha.
\end{split}
\end{align}
Since the true distribution of $(X, Y)$ is unknown, we approximate this objective using the observed dataset $\{(x_i, y_i)\}_{i=1}^{n}$. Replacing the expectation in~\eqref{eq:population:min:vol} with an empirical mean and enforcing the coverage constraint empirically leads to the following M-estimation problem:
\begin{align}
\label{eq:ERM:min:vol}
\begin{split}
    \min \quad & \frac{1}{n} \sum_{i=1}^{n} \text{Vol}\left( \mathbb{B}(p, M(x_i), f_\theta(x_i)) \right) \\
    \mathrm{s.t.} \quad & M(\cdot) \geqm 0, \; p > 0,\; \theta \in \Theta, \\
    & \mathrm{Card} \left\{ i \in [n] \mid \|M(x_i) (y_i - f_{\theta}(x_i))\|_p \leq 1 \right\} \geq n - r + 1.
\end{split}
\end{align}
Here, $r$ is chosen such that at least $1 - \alpha$ of the training points are contained within their respective prediction sets. This formulation minimizes the average volume of the prediction sets while ensuring empirical coverage. 

We now derive an equivalent reformulation that facilitates efficient optimization. The following proposition formalizes the extension of the MVCS framework from formulation~\eqref{eq:min:vol:regression-theta} to locally adaptive transformations, allowing the prediction set to vary with $x$. While we present the result for a single norm $p$, this formulation can be immediately extended to multiple norms in the spirit of Proposition~\ref{prop:min:vol:multi}. {\color{black}However, our numerical experiments indicate that extending the multi-norm MVCS framework to locally adaptive transformations in regression often leads to overfitting. While multi-norm MVCS effectively captures complex uncertainty structures in Section~\ref{sec:min:vol:cov:shape}, applying it in a regression setting requires learning feature-dependent transformations, which substantially increases the complexity of the optimization landscape. As a result, the learned prediction sets may fit well for certain feature values $x_i$ but fail to generalize across the input space, leading to inconsistent adaptivity. This suggests that while multi-norm formulations offer greater flexibility, their practical success in regression depends on having sufficient training data and appropriate regularization to mitigate overfitting.}

\begin{figure}[t]
    \center
    \subfigure{\includegraphics[width=0.48\columnwidth]{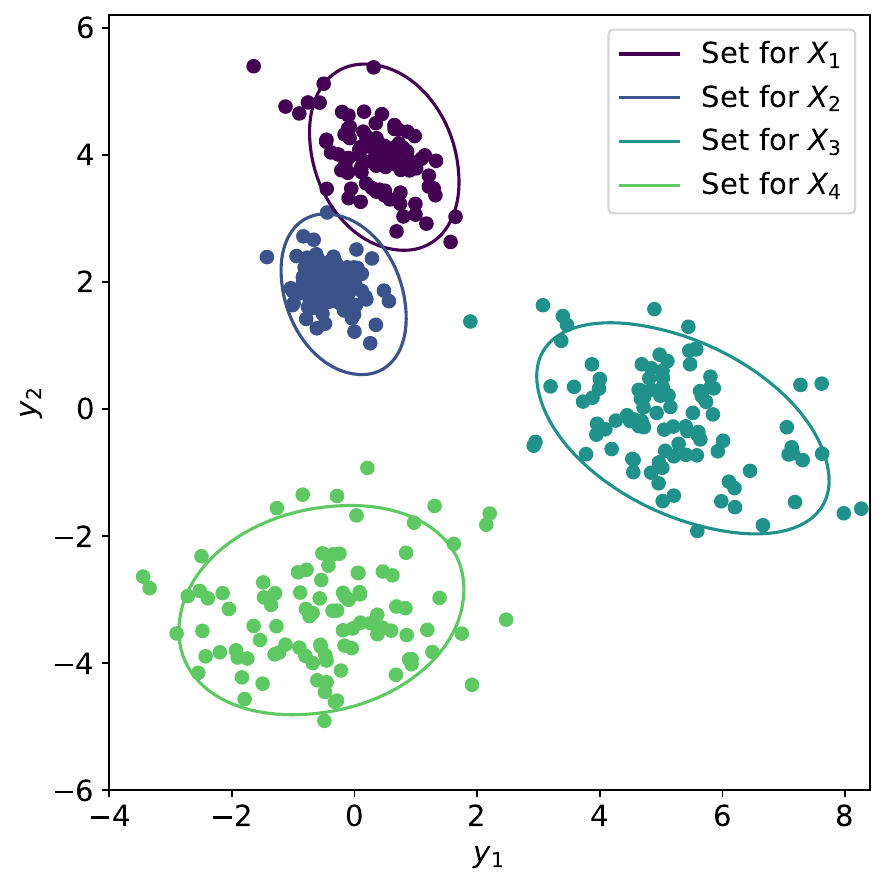}} ~
    \subfigure{\includegraphics[width=0.48\columnwidth]{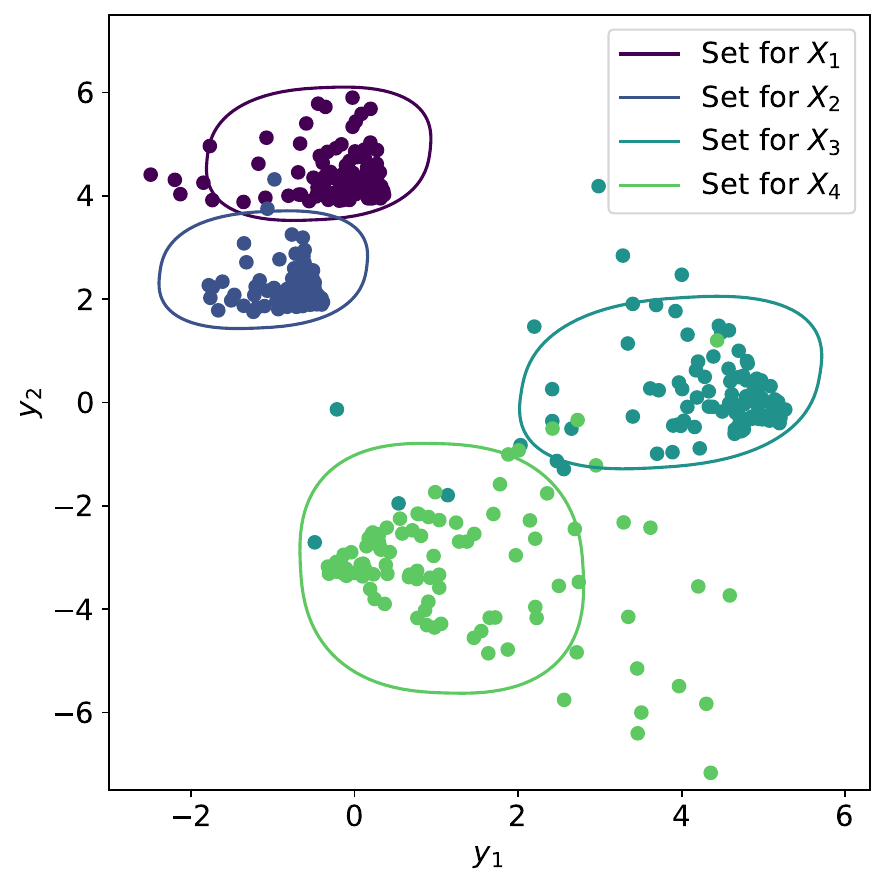}} 
    \vspace{-5mm}
    \caption{MVCSs obtained via the locally adaptive formulation~\eqref{eq:min:vol:exact:adaptive}. Each shape corresponds to the prediction set at a different feature value $X_i$. (Left) Gaussian residuals, where the learned $p=2.08$. (Right) Exponential residuals, where the optimization selects $p=2.52$. Compared to Figure~\ref{figure:learn:theta}, where a single global transformation was applied across all $X$, these results illustrate how local adaptivity allows the prediction sets to vary with $X$, capturing feature-dependent uncertainty more effectively.}
    \label{figure:learn:adaptive}
\end{figure}

\begin{proposition}[Locally adaptive MVCS]
\label{prop:min:vol:adaptive}
The locally adaptive MVCS problem~\eqref{eq:ERM:min:vol} is equivalent to the following optimization problem:
\begin{align}
\label{eq:min:vol:exact:adaptive}
\begin{split}
    \min \quad &  \log \left(\sum_{i=1}^{n} \frac{1}{\det (\Lambda(x_i))} \right) + k \log \sigma_r\left\{ \|\Lambda(x_i) (y_i - f_\theta(x_i)) \|_p \right\} + \log \lambda(B_{\|\cdot \|_p}(1)) \\
    \mathrm{s.t.} \quad & \Lambda(\cdot) \geqm 0, \; p>0, \; \theta \in \Theta,
\end{split}
\end{align}
and the optimal $M^\star(x)$ in \eqref{eq:ERM:min:vol} can be recovered from the optimal solution $\Lambda^\star(x)$ of \eqref{eq:min:vol:exact:adaptive} as
\begin{align*}
    M^\star(x) = \sigma_r\left\{ \|\Lambda^\star(x_i) (y_i - f_\theta(x_i)) \|_{p^\star} \right\}^{-1} \Lambda^\star(x).
\end{align*}
\end{proposition}
\begin{proof}
Using the change of coordinates $\Lambda \coloneqq \nu M$, for $\nu \geq 0$, we have that problem~\eqref{eq:ERM:min:vol} is equivalent to
\begin{align}
\label{eq:min:vol:exact:adaptive:temp:1}
\begin{split}
\min \quad & k \log \nu + \log \left( \sum_{j=1}^m {\lambda(B_{\|\cdot \|_{p}}(1))}{\det(\Lambda(x_i))}^{-1}\right) \\
\mathrm{s.t.} \quad & \Lambda(\cdot) \geqm 0, \; p>0, \; \theta \in \Theta,\; \nu \geq 0 \\
& \mathrm{Card} \left\{ i \in [n] \mid \|\Lambda(x_i) (y_i - f_{\theta}(x_i))\|_p \leq \nu \right\} \geq n - r + 1,
\end{split}
\end{align}
where we have used the expression of the volume provided in \eqref{eq:volume}. Now, due the cardinality constraint, for given $\Lambda, p, \theta$, the optimal $\nu^\star(\Lambda, p, \theta)$ is equal to the $(n-r+1)$-th smallest value in the set $\left\{ \|\Lambda(x_i) (y_i - f_{\theta}(x_i))\|_p \right\}_{i=1}^n$, or, equivalently, the $r$-th largest value, i.e., $\nu^\star = \sigma_r\left\{ \|\Lambda(x_i) (y_i - f_{\theta}(x_i))\|_p \right\}$. Plugging this into \eqref{eq:min:vol:exact:adaptive:temp:1}, we obtain problem \eqref{eq:min:vol:exact:adaptive}.
\end{proof}

The first term in the objective function ensures that the \emph{average volume} of the prediction sets across all feature values is minimized, rather than optimizing a single global shape. This accounts for varying uncertainty levels throughout the feature space. The second term enforces the coverage constraint by selecting the $r$-th largest nonconformity score, ensuring that at least $n - r + 1$ points remain within their respective sets. 

\textcolor{black}{
\begin{remark}[Unconstrained locally adaptive MVCS]
Extending the single-norm formulation from Section~\ref{sec:min:vol:cov:shape}, we express the objective function in an unconstrained form as:
\begin{equation}
\label{eq:loss:MVCS:adaptive}
 \log \left(\sum_{i=1}^{n} \frac{1}{\det (A(x_i)A(x_i)^\top)} \right) \\
+ k \log \sigma_r\left\{ \|A(x_i) A(x_i)^\top (y_i + f_\theta(x_i) \|_{|p|} \right\} \\
+ \log \lambda(B_{\|\cdot \|_{|p|}}(1)),
\end{equation}
where $A(\cdot)$ is a function mapping features $x \in \mathbb{R}^d$ to a transformation matrix in $\mathbb{R}^{k\times k}$, ensuring that the learned transformation remains positive semidefinite via the parameterization $\Lambda(x) = A(x) A(x)^\top$. Moreover, the reparameterization $|p|$ allows for unconstrained optimization over the norm parameter, enabling updates without restrictions on its sign. This formulation facilitates efficient first-order gradient-based optimization without requiring explicit constraints on $\Lambda(\cdot)$ or $p$.
\end{remark}}

While optimizing over $A(\cdot)$ introduces additional computational complexity, it provides a significant advantage: the resulting prediction sets are tighter and better calibrated to local uncertainty. Empirical results confirm that this adaptivity leads to meaningful improvements, particularly in settings where uncertainty varies across feature space. We illustrate this in Figure~\ref{figure:learn:adaptive}, where MVCSs are constructed in a locally adaptive manner. Compared to Figure~\ref{figure:learn:theta}, where a single transformation was applied globally, the prediction sets here adjust to the residual distribution at each feature value $x_i$. The left panel, corresponding to Gaussian residuals, yields $p=2.08$, while the right panel, with exponential residuals, results in $p=2.52$. By allowing the prediction sets to vary with $X$, this approach more effectively captures heteroskedasticity, improving both calibration and efficiency.

%------------------------------------------------------------------------
%------------------------------------------------------------------------
%------------------------------------------------------------------------

\section{Conformalized Minimum-Volume Prediction Sets}
\label{sec:conformalizing:MVVCS}

Thus far, we have developed a framework for constructing minimum-volume prediction sets by leveraging geometric structure in residuals. These sets are designed to provide high coverage within the training data, and we need to ensure that the coverage properties extend to unseen test samples. We do this by integrating the methods discussed thus far with conformal prediction. 

\textcolor{black}{The conformalization procedure we adopt follows the standard split conformal prediction framework \cite{vovk2005algorithmic, papadopoulos2002inductive}, with the key difference being our choice of nonconformity scores, which are derived from the locally adaptive MVCS sets presented in Section~\ref{sec:multivariate:regression}.} We consider a setting where we are given a dataset of $n$ i.i.d.\ samples $(X_i, Y_i) \sim \mathbb P$, where $\mathbb P$ is the unknown joint distribution of covariates $X_i \in \mathbb{R}^d$ and responses $Y_i \in \mathbb{R}^k$. Our goal is to construct covariate-dependent prediction sets $C(x)$ that satisfy the coverage guarantee
\begin{align*}
    \mathrm{Prob} \left\{ Y \in C(X) \right\} \geq 1 - \alpha,
\end{align*}
for a prescribed confidence level $1 - \alpha$. In line with the split conformal prediction framework, we partition the dataset into \emph{independent} sets:
\begin{itemize}
    \item $\mathcal{D}_1$, the \emph{proper training set}, with $\card(\mathcal{D}_1) = n_1$;
    \item $\mathcal{D}_2$, the \emph{calibration set}, with $\card(\mathcal{D}_2) = n_2$.
\end{itemize}
Using $\mathcal{D}_1$, we fit a predictive model $f_\theta$ and estimate a minimum-volume transformation function $M(x)$ using the optimization formulation presented in Proposition~\ref{prop:min:vol:adaptive}. The choice of structure—whether to use a fixed norm, a single learned $p$-norm, or a multiple-norm extension—determines the learned transformation and the associated prediction set shape. For simplicity, we present a version of the conformalization procedure in which the prediction sets are derived using a single learned $p$-norm, though the same approach extends to arbitrary norm choices and multi-norm formulations. \textcolor{black}{Thus, while the conformalization step follows the standard framework, our methodology is distinguished by the use of locally adaptive MVCS-based scores.} In this setting, we define the nonconformity scores as
\begin{align*} 
    s(x, y) \coloneqq \|M(x) (y - f_\theta(x))\|_p.
\end{align*}
To simplify notation, we denote the scores computed on the calibration set as $S_i = s(X_i, Y_i)$, for $i \in [n_2]$. These scores are then used to determine a conformalized threshold that ensures the required coverage. Specifically, we compute an empirical quantile of the calibration scores, incorporating a finite-sample correction. The threshold is given by
\begin{align*}
    \widehat{q}_\alpha = \lceil (1 - \alpha)(n_2 + 1) \rceil \text{-smallest value of } S_i, \text{ for } i \in [n_2].
\end{align*}
The quantity $\widehat{q}_\alpha$ defines the conformalized radius that determines the final prediction region. By leveraging the independent calibration dataset, this approach guarantees that the constructed prediction sets satisfy the marginal coverage requirement in a finite-sample sense. For a new test point $X_{n+1}$, we define the conformalized prediction set as
\begin{align}
\label{eq:conformalized-prediction}
    C(X_{n+1}) \coloneqq \left\{ y \in \mathbb{R}^k \mid s(X_{n+1},y) \leq \widehat{q}_\alpha \right\} = \left\{ y \in \mathbb{R}^k \mid \|M(X_{n+1}) (y - f_\theta(X_{n+1}))\|_p \leq \widehat{q}_\alpha \right\}.
\end{align}
This formulation maintains the same geometric structure as the minimum-volume prediction sets derived in Section~\ref{sec:multivariate:regression}, but now scales their size using $\widehat{q}_\alpha$ to guarantee marginal coverage. We now formally state the finite-sample coverage property of the conformalized MVCS.

\begin{lemma}[MVCS finite-sample guarantee]
\label{prop:conformal-coverage}
Let $(X_{n+1}, Y_{n+1})$ be a test point sampled from $\mathbb P$, independent of the calibration samples $(X_1, Y_1), \dots, (X_{n_2}, Y_{n_2})$. Then, the prediction set $C(X_{n+1})$ defined in \eqref{eq:conformalized-prediction} satisfies\footnote{Assuming that there are almost surely no ties among the scores $\{S_i\}_{i=1}^{n_2} \cup \{S_{n+1}\}$.}
\begin{align*}
    \mathrm{Prob} \left\{ Y_{n+1} \in C(X_{n+1}) \mid \{(X_i, Y_i)\}_{i \in \mathcal{D}_1} \right\} \in \left[ 1 - \alpha, 1 - \alpha + \frac{1}{n_2+1} \right).
\end{align*}
\end{lemma}

The proof of Lemma~\ref{prop:conformal-coverage} follows directly from standard conformal prediction arguments. Given that the calibration scores $\{S_i\}_{i \in \mathcal{D}_2}$ and the test score $S_{n+1}$ are i.i.d., the probability that $S_{n+1}$ is at most the $\lceil (1 - \alpha)(n_2 + 1) \rceil$-th smallest calibration score follows a standard rank argument, yielding the desired result. This conformalization procedure ensures that the minimum-volume prediction sets constructed in Section~\ref{sec:multivariate:regression} achieve valid marginal coverage guarantees without requiring additional parametric assumptions on the residual distribution. 

While conformalization provides finite-sample coverage guarantees, it does require a separate calibration dataset, as the conformalized radius is determined independently from the training process. However, this tradeoff is often well-justified:
\begin{itemize}
    \item \emph{Improved generalization}: Conformalized MVCSs prevent overfitting to training data, ensuring reliable uncertainty quantification.
    \item \emph{Minimal computational overhead}: The conformal step only requires computing empirical quantiles, making it computationally efficient.
    \item \emph{Applicability to any MVCS formulation}: The method is fully compatible with any choice of norm structure (fixed-norm, single-norm, or multi-norm).
\end{itemize}
In practice, the benefits of conformalization are observed in improved test coverage and well-calibrated prediction sets. The next section presents empirical results demonstrating these advantages across various regression tasks.

%------------------------------------------------------------------------
%------------------------------------------------------------------------
%------------------------------------------------------------------------

\section{Numerical Validation}
\label{sec:numerical:validation}

All the results presented in this section are fully reproducible and available in the associated GitHub repository.\footnote{\url{https://github.com/ElSacho/MVCS}}

%------------------------------------------------------------------------

\subsection{Training procedure}
\label{subsec:training:procedure}

Optimizing the training objective \eqref{eq:min:vol:exact:adaptive} is challenging due to its nonconvexity. To enable efficient minimization using first-order optimization techniques, we introduce a structured parameterization. Specifically, both the predictive function for the center $f_\theta$ and the transformation matrix $\Lambda_\phi$ are modeled using neural networks. To ensure positive semi-definiteness of the learned transformation, we define a mapping $A_{\phi} : \mathbb{R}^d \to \mathbb{R}^{k \times k}$ that outputs a $k \times k$ matrix and set $\Lambda_{\phi}(x) = A_{\phi}(x) A_{\phi}(x)^T$. This guarantees that $\Lambda_{\phi}(x)$ satisfies the constraints of the optimization problem while allowing for flexible, data-driven adaptation. {\color{black}This parameterization enables direct (unconstrained) minimization of the loss in~\eqref{eq:loss:MVCS:adaptive}.} Our training procedure consists of three sequential stages:

\begin{itemize}
    \item \emph{Pretraining the predictive model.} We first train the predictive model $f_{\theta}$, which determines the center of the prediction set, using a mean squared error objective. This pretraining step stabilizes the learning process and improves convergence.
    
    \item \emph{Optimizing the transformation matrix.} Keeping $f_{\theta}$ fixed, we then optimize the matrix model by adjusting $\Lambda_{\phi}(x)$ based on the residuals. This sequential approach is motivated by empirical observations: when $f_{\theta}$ overfits, the residual structure deteriorates, making it difficult for the matrix model to generalize effectively.
    
    \item \emph{Joint optimization.} Finally, we jointly optimize both $\theta$ and $\phi$, refining the predictive model and the uncertainty quantification simultaneously. To approximate the coverage constraint efficiently in mini-batch training, we compute the $r$-th largest nonconformity score within each batch, where $r = \lfloor \alpha B \rfloor$ and $B$ denotes the batch size.
\end{itemize}

This structured approach ensures a balance between expressivity and stability, preventing overfitting while effectively capturing the heteroskedastic structure of the data. The resulting prediction sets are both adaptive and well-calibrated, as demonstrated in our empirical evaluations.

%------------------------------------------------------------------------

\subsection{Comparative baselines}

We compare our method (single-norm MVCS) against three established baselines: Naïve Quantile Regression \cite{dheur2025multi} (naïve QR), the empirical covariance matrix \cite{johnstone2021conformal} (emp.\ cov.), and the local empirical covariance method \cite{messoudi2022ellipsoidal} (loc.\ emp.\ cov.). Below, we summarize these approaches.

\begin{itemize}
    \item \emph{Naïve quantile regression.} This strategy involves performing quantile regression using the pinball loss \cite{Steinwart_2011} for the quantiles ${\tilde{\alpha}}/{2}$ and $1-{\tilde{\alpha}}/{2}$ along all axes on a training dataset. To avoid overcovering, we select $\tilde{\alpha} = 2(1 - (1 - \alpha)^{1/k})$, as suggested by previous work \cite{dheur2025multi} to fit the quantile networks $q$. The prediction regions are then conformalized on a calibration dataset using the score
    \begin{align*}
        s(x, y) = \max_{i\in [k]} \{ q_{{\tilde{\alpha}}/{2}}(x)_i - y_i, y_i - q_{1 - {\tilde{\alpha}}/{2}}(x)_i \},
    \end{align*}
    obtaining the calibration quantile $\hat{q}^p_\alpha$. At test time, the coverage and average volume of the region are computed on a test dataset, with the volume for each test point $x$ given by $\prod_i (q_{1 - {\tilde{\alpha}}/{2}}(x)_i - q_{{\tilde{\alpha}}/{2}}(x)_i + 2\hat{q}^p_\alpha)$.

    \item \emph{Empirical covariance matrix.} A widely used alternative constructs ellipsoidal prediction regions by leveraging the empirical covariance structure of the residuals \cite{johnstone2021conformal}. The estimated covariance matrix $\hat{\Sigma}$ is computed from the training residuals, and the prediction region is shaped using the Mahalanobis norm. The nonconformity score is given by
    \begin{align*}
        s(x, y) := \|\hat{\Sigma}^{-1/2}(y - f_{\theta}(x))\|_2.
    \end{align*}
    We compute the calibration quantile $\hat{q}^e_\alpha$ with the calibration dataset and the volume of the ellipsoid is $\lambda(B_{\|\cdot \|_{2}}(1)) \cdot (\det((\hat{q}^e_\alpha)^2 \cdot \hat{\Sigma}))^{1/2}$.

    \item \emph{Local  empirical covariance matrix.} While the global covariance structure captures overall variability, it fails to adapt to local heteroskedasticity. To address this limitation, a local empirical covariance approach \cite{messoudi2022ellipsoidal} estimates $\hat{\Sigma}_x$ using only the $m$ nearest neighbors of $x$ in the training dataset. The nonconformity score is then computed as
    \begin{align*}
        s(x, y) := \|\hat{\Sigma}^{-1/2}_x(y - f_{\theta}(x))\|_2.
    \end{align*}
    We compute the calibration quantile $\hat{q}^l_\alpha$ with the calibration dataset and the corresponding volume for a test point $x$ is $\lambda(B_{\|\cdot \|_{2}}(1)) \cdot (\det((\hat{q}^l_\alpha)^2 \cdot \hat{\Sigma}_x))^{1/2}$. By locally adapting the covariance structure, this method yields more flexible prediction sets that better reflect data-dependent uncertainty.
\end{itemize}

These competing strategies serve as baselines for evaluating our proposed approach, demonstrating its advantages in both adaptivity and efficiency.

%------------------------------------------------------------------------

\subsection{Synthetic dataset}

Throughout our experiments, we utilize four datasets: a training dataset, a validation dataset for model selection, a calibration dataset for conformalization, and a test dataset for evaluating final coverage and volume. This setup follows standard practice in the literature.

\paragraph{Conditional coverage.}
A key consideration in conformal prediction is whether a given strategy maintains conditional coverage. While marginal coverage guarantees ensure that $\mathrm{Prob}\left\{Y \in C(X)\right\} \geq 1 - \alpha$, conditional validity (i.e., requiring $\mathrm{Prob}\{Y \in C(X) \mid X\} = 1 - \alpha$) is generally unattainable without strong assumptions on the data distribution \cite{foygel2021limits}. Aggressively minimizing volume can degrade conditional validity \cite{huang2023uncertainty}, as our optimization objective \eqref{eq:min:vol:exact:adaptive} does not explicitly enforce constraints on the $r$-largest values. However, the neural network parameterization promotes generalization, leading to empirically favorable results despite the lack of theoretical convergence guarantees due to nonconvexity.

To assess conditional coverage empirically, we conduct a 1D regression experiment with $Y = f(X) + t(X) B$, where $(X, Y) \in \mathbb{R}^2$ and $B \sim \mathcal{E}(1)$ follows an asymmetric exponential distribution scaled by $t(X) = 0.5 + 2X$. We compare our method against the naïve quantile regression strategy. Figure~\ref{figure:conditional:1D} reports results for a target coverage of $0.99$ (additional results for $0.90$ are in Appendix Figure~\ref{fig:conditional:1d:09}). We observe that our adaptive framework effectively adjusts to the data, showing no clear signs of degraded conditional coverage.

\begin{figure}[t]
    \center
    \subfigure{\includegraphics[width=0.48\columnwidth]{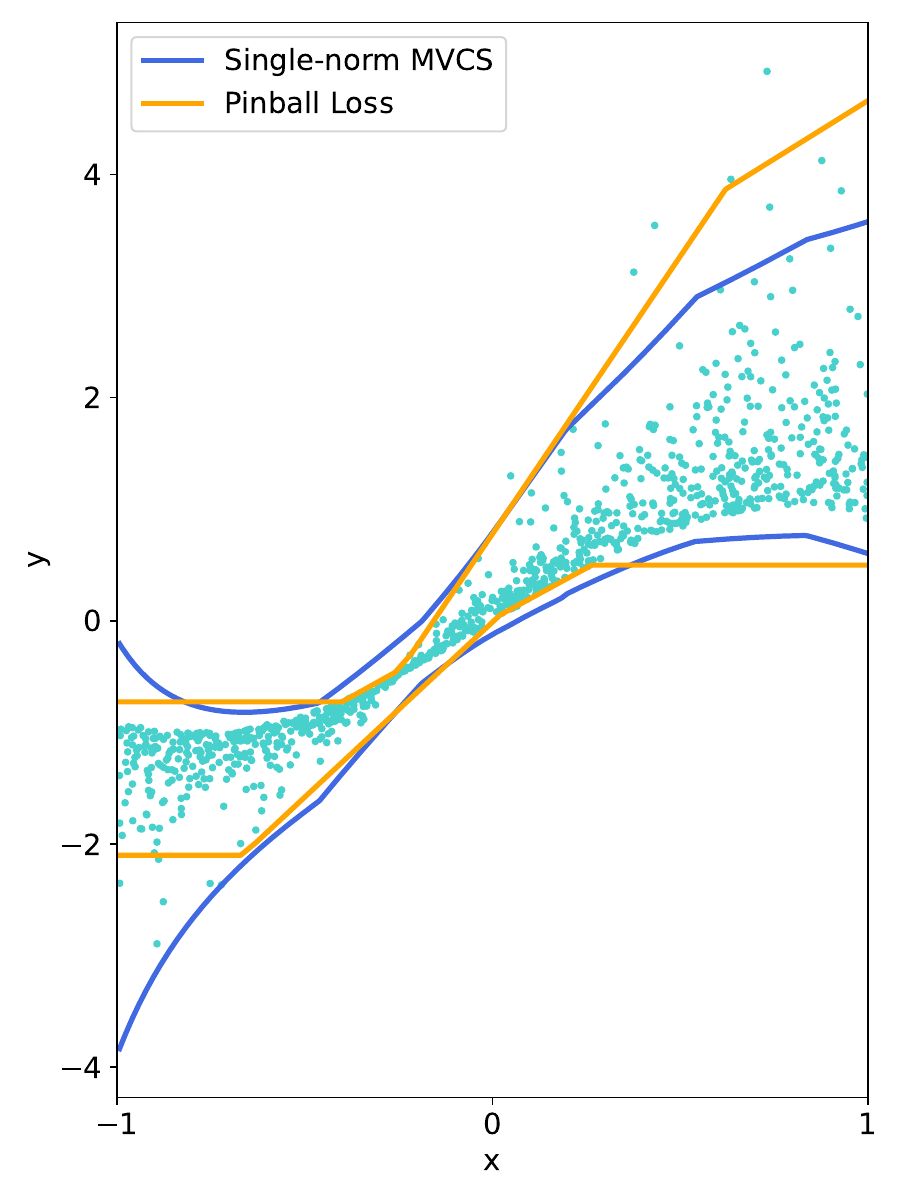}} ~
    \subfigure{\includegraphics[width=0.48\columnwidth]{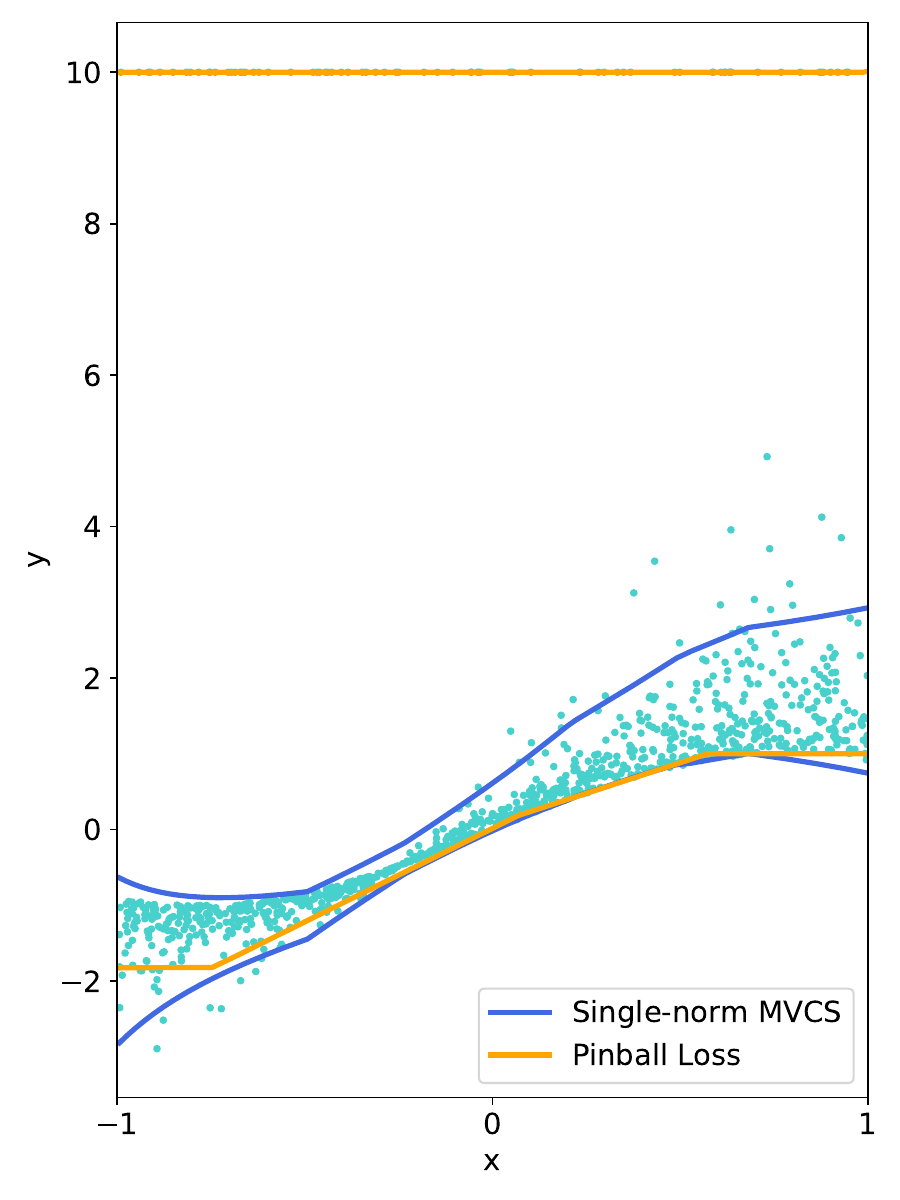}} 
    \vspace{-5mm}
    \caption{(Left) Conformal prediction sets for a target coverage level of $0.99$. (Right) Conformal prediction sets for a target coverage level of $0.90$, with $3\alpha/4$ of the data consisting of outliers.}
    \label{figure:conditional:1D}
\end{figure}

\paragraph{Robustness to outliers.}
Our loss function is inherently robust to outliers, up to a fraction of $\alpha$. To illustrate this, we replicate the previous experiment while introducing extreme values for a fraction $3\alpha/4$ of the data points.\footnote{These extreme values are fixed at 10.} We compare our approach with the pinball loss and the standard nonconformity score, using the same evaluation framework. The pinball loss is trained on the quantiles $\alpha/2$ and $1 - \alpha/2$, but under contamination, its optimal quantiles shift, leading to unreliable prediction sets. In particular, when the proportion of outliers in the upper quantile exceeds $\alpha/2$, the estimated $(1 - \alpha)$-quantile is dominated by extreme values, compromising the validity of the prediction sets. This highlights the limitations of quantile-based methods in the presence of outliers, whereas our approach remains robust, maintaining stable and reliable uncertainty quantification.

\paragraph{Multivariate regression.}
We evaluate our method’s performance in multivariate regression by comparing the average volume and achieved coverage of the prediction sets on a test dataset against baseline approaches. The data is generated as $Y = f(X) + t(X) B$, where $X \sim \mathcal{N}(0, I_4)$ and $Y \in \mathbb{R}^4$. The noise term $B$ follows a fixed distribution and is modulated by the transformation $t(X)$, with further details provided in Appendix~\ref{app:synthetic:datageneration}. 

We assess performance across four synthetic settings: (i) fixed exponential noise ($t$ constant), (ii) exponential noise with a variable transformation $t(X)$, (iii) fixed Gaussian noise, and (iv) Gaussian noise with a transformation. The volume of each prediction set is computed as the mean volume across test points. To account for the effect of dimensionality, we normalize each volume by taking its $1/k$ power. Coverage is measured as the proportion of test samples contained within their respective sets. The empirical covariance baseline corresponds to the predictor with the lowest Mean Squared Error (MSE), selected before fine-tuning with our proposed loss. Results are reported for a target coverage of $0.90$, averaged over 10 runs, excluding the highest and lowest volumes. In all cases, our method achieves the smallest average volume (Table~\ref{tab:volume:synthetic:0.9}) while maintaining valid marginal coverage (Table~\ref{tab:coverage:synthetic:0.9}). Additional results for a coverage level of $0.99$ are provided in Appendix~\ref{app:synthetic:additional:results}.

\begin{table}[t]
\centering
\begin{tabular}{@{}lcccc@{}}
\toprule
Dataset & Naïve QR & Emp.\ Cov.\ & Loc.\ Emp.\ Cov.\ & MVCS \\
\midrule
Exp. Fixed  & $7.63 \pm 0.11$  & $6.57 \pm 0.12$  & $6.45 \pm 0.14$  & $\textbf{6.00} \pm 0.09$ \\ \hline
Exp. Transformed & $ 10.66 \pm 0.13$  & $9.46 \pm 0.19$  & $9.40 \pm 0.17$  & $\textbf{8.45} \pm 0.09$ \\ \hline
Gau. Fixed  & $ 6.63 \pm 0.09$  & $5.43 \pm 0.09$  & $5.34 \pm 0.08$  & $\textbf{5.11} \pm 0.10$ \\ \hline
Gau. Transformed  & $ 9.21 \pm 0.06$  & $8.39 \pm 0.05$  & $7.97 \pm 0.09$  & $\textbf{7.43} \pm 0.06$ \\ \hline
\bottomrule
\end{tabular}
\caption{Table for Normalized Volume (desired coverage 0.9).}
\label{tab:volume:synthetic:0.9}
\end{table}

\begin{table}[t]
\centering
\begin{tabular}{@{}lcccc@{}}
\toprule
Dataset & Naïve QR & Emp.\ Cov.\ & Loc.\ Emp.\ Cov.\ & MVCS \\
\midrule
Exp. Fixed    & $89.6 \pm 0.7$  & $89.7 \pm 0.6$  & $89.8 \pm 0.5$  & $89.7 \pm 0.6$ \\ \hline
Exp. Transformed   & $89.6 \pm 0.7$  & $89.9 \pm 0.7$  & $89.9 \pm 0.7$  & $90.0 \pm 0.8$ \\ \hline
Gau. Fixed    & $89.9 \pm 0.6$  & $90.0 \pm 0.5$  & $89.8 \pm 0.6$  & $89.7 \pm 0.7$ \\ \hline
Gau. Transformed    & $89.8 \pm 0.5$  & $90.0 \pm 0.4$  & $89.8 \pm 0.4$  & $89.7 \pm 0.3$ \\ \hline
\bottomrule
\end{tabular}
\caption{Table for Coverage in \% (desired coverage 0.9).}
\label{tab:coverage:synthetic:0.9}
\end{table}

%------------------------------------------------------------------------

\subsection{Real datasets}

We evaluate our approach on nine benchmark datasets commonly used in multivariate regression studies \cite{dheur2025multi, feldman2023calibrated, wang2023probabilistic}. Details regarding these datasets, including the response dimensionality and hyperparameter choices, are provided in Appendix~\ref{app:dataset:information}. We compare different strategies for both $\alpha = 0.1$ and $\alpha = 0.01$. Each dataset is split into four subsets—training, validation, calibration, and test—using respective proportions of 70\%, 10\%, 10\%, and 10\%.\footnote{Except for the energy dataset, for which the data splitting is 55\%, 15\%, 15\%, 15\% due to the smaller amount of samples.} To ensure robust evaluation, results are averaged over 10 independent runs, discarding the highest and lowest volume values. Additionally, each volume is scaled by taking its $1/k$-th power to normalize for dimensionality.

To preprocess the data, we apply a quantile transformation to both the covariates ($X$) and responses ($Y$) using \texttt{scikit-learn} \cite{pedregosa2011scikit}, ensuring normalization across datasets. Our approach consistently minimizes the volume of prediction sets while maintaining valid marginal coverage. Specifically, we achieve the best average volume in sixteen out of eighteen experiments while preserving the desired coverage level (Tables~\ref{tab:coverage:real:0.9} and~\ref{tab:coverage:real:0.99}).

\begin{table}[t]
\centering
\begin{tabular}{@{}lcccc@{}}
\toprule
Dataset & Naïve QR & Emp.\ Cov.\ & Loc.\ Emp.\ Cov.\ & MVCS \\
\midrule
Bias correction & $1.29 \pm 0.02$  & $\textbf{1.26} \pm 0.03$  & $1.45 \pm 0.10$  & $1.33 \pm 0.24$ \\ \hline
CASP & $1.40 \pm 0.01$  & $1.52 \pm 0.02$  & $1.44 \pm 0.02$  & $\textbf{1.32} \pm 0.02$ \\ \hline
Energy & $1.28 \pm 0.11$  & $1.10 \pm 0.16$  & $1.10 \pm 0.16$  & $\textbf{0.97} \pm 0.13$ \\ \hline
House & $1.37 \pm 0.02$  & $1.39 \pm 0.02$  & $1.38 \pm 0.02$  & $\textbf{1.33} \pm 0.02$ \\ \hline
rf1 & $0.43 \pm 0.02$  & $0.44 \pm 0.02$  & $0.64 \pm 0.03$  & $\textbf{0.39} \pm 0.05$ \\ \hline
rf2 & $0.61 \pm 0.01$  & $0.42 \pm 0.02$  & $0.44 \pm 0.02$  & $\textbf{0.35} \pm 0.01$ \\ \hline
scm1d & $2.71 \pm 0.09$  & $1.74 \pm 0.06$  & $1.74 \pm 0.06$  & $\textbf{1.47} \pm 0.08$ \\ \hline
scm20d & $3.45 \pm 0.47$  & $2.64 \pm 0.49$  & $2.64 \pm 0.49$  & $\textbf{1.51} \pm 0.03$ \\ \hline
Taxi & $3.48 \pm 0.02$  & $3.42 \pm 0.04$  & $3.35 \pm 0.03$  & $\textbf{3.18} \pm 0.02$ \\ \hline
\bottomrule
\end{tabular}
\caption{Table for Normalized Volume (desired coverage 0.9).}
\label{tab:volume:real:0.9}
\end{table}

\begin{table}[t]
\centering
\begin{tabular}{@{}lcccc@{}}
\toprule
Dataset & Naïve QR & Emp.\ Cov.\ & Loc.\ Emp.\ Cov.\ & MVCS \\
\midrule
Bias correction  & $90.5 \pm 1.1$  & $90.2 \pm 1.3$  & $90.1 \pm 1.0$  & $90.3 \pm 0.8$ \\ \hline
CASP  & $89.9 \pm 0.5$  & $90.1 \pm 0.3$  & $89.9 \pm 0.4$  & $90.1 \pm 0.4$ \\ \hline
Energy  & $89.3 \pm 2.6$  & $91.4 \pm 3.8$  & $91.4 \pm 3.8$  & $90.7 \pm 2.3$ \\ \hline
House  & $90.3 \pm 0.8$  & $90.3 \pm 0.6$  & $90.6 \pm 0.7$  & $90.3 \pm 0.7$ \\ \hline
rf1  & $89.9 \pm 1.0$  & $89.3 \pm 1.8$  & $89.3 \pm 1.1$  & $89.4 \pm 1.4$ \\ \hline
rf2  & $89.8 \pm 0.7$  & $90.1 \pm 1.2$  & $90.3 \pm 1.1$  & $89.9 \pm 1.0$ \\ \hline
scm1d & $90.9 \pm 0.8$  & $90.3 \pm 1.0$  & $90.3 \pm 1.0$  & $91.0 \pm 1.2$ \\ \hline
scm20d & $90.4 \pm 0.8$  & $90.2 \pm 0.7$  & $90.2 \pm 0.7$  & $90.3 \pm 0.8$ \\ \hline
Taxi & $90.0 \pm 0.3$  & $90.0 \pm 0.3$  & $89.9 \pm 0.4$  & $90.0 \pm 0.3$ \\ \hline
\bottomrule
\end{tabular}
\caption{Table for Coverage in \% (desired coverage 0.9).}
\label{tab:coverage:real:0.9}
\end{table}

\begin{table}[t]
\centering
\begin{tabular}{@{}lcccc@{}}
\toprule
Dataset & Naïve QR & Emp.\ Cov.\ & Loc.\ Emp.\ Cov.\ & MVCS \\
\midrule
Bias correction  & $2.25 \pm 0.08$  & $2.38 \pm 0.25$  & $3.13 \pm 0.38$  & $\textbf{2.21} \pm 0.36$ \\ \hline
CASP  & $3.41 \pm 0.05$  & $3.73 \pm 0.21$  & $3.79 \pm 0.36$  & $\textbf{2.94} \pm 0.11$ \\ \hline
Energy  & $3.45 \pm 1.14$  & $4.19 \pm 1.29$  & $4.19 \pm 1.29$  & $\textbf{2.85} \pm 1.51$ \\ \hline
House  & $2.67 \pm 0.09$  & $3.03 \pm 0.14$  & $3.25 \pm 0.19$  & $\textbf{2.29} \pm 0.07$ \\ \hline
rf1  & $1.17 \pm 0.13$  & $1.49 \pm 0.38$  & $1.55 \pm 0.24$  & $\textbf{0.97} \pm 0.16$ \\ \hline
rf2  & $1.67 \pm 0.31$  & $1.59 \pm 0.46$  & $1.30 \pm 0.24$  & $\textbf{0.93} \pm 0.30$ \\ \hline
scm1d  & $5.06 \pm 0.12$  & $\textbf{4.43} \pm 0.60$  & $\textbf{4.43} \pm 0.60$  & $4.54 \pm 0.78$ \\ \hline
scm20d  & $6.56 \pm 0.40$  & $4.63 \pm 0.26$  & $4.63 \pm 0.26$  & $\textbf{4.25} \pm 0.27$ \\ \hline
Taxi  & $5.33 \pm 0.09$  & $6.19 \pm 0.10$  & $5.61 \pm 0.07$  & $\textbf{5.15} \pm 0.06$ \\ \hline
\bottomrule
\end{tabular}
\caption{Table for Normalized Volume (desired coverage 0.99).}
\label{tab:volume:real:0.99}
\end{table}

\begin{table}[t]
\centering
\begin{tabular}{@{}lcccc@{}}
\toprule
Dataset & Naïve QR & Emp.\ Cov.\ & Loc.\ Emp.\ Cov.\ & MVCS \\
\midrule
Bias correction   & $99.1 \pm 0.2$  & $99.3 \pm 0.2$  & $99.1 \pm 0.4$  & $99.4 \pm 0.4$ \\ \hline
CASP   & $99.2 \pm 0.1$  & $99.0 \pm 0.1$  & $99.0 \pm 0.1$  & $99.1 \pm 0.1$ \\ \hline
Energy   & $99.8 \pm 0.4$  & $99.8 \pm 0.4$  & $99.8 \pm 0.4$  & $99.4 \pm 0.8$ \\ \hline
House   & $99.1 \pm 0.2$  & $99.0 \pm 0.3$  & $99.1 \pm 0.2$  & $99.0 \pm 0.2$ \\ \hline
rf1   & $99.0 \pm 0.3$  & $99.0 \pm 0.3$  & $99.0 \pm 0.5$  & $99.0 \pm 0.3$ \\ \hline
rf2   & $99.0 \pm 0.4$  & $99.0 \pm 0.5$  & $99.1 \pm 0.1$  & $99.2 \pm 0.3$ \\ \hline
scm1d  & $99.2 \pm 0.3$  & $99.3 \pm 0.3$  & $99.3 \pm 0.3$  & $99.3 \pm 0.3$ \\ \hline
scm20d  & $99.1 \pm 0.2$  & $99.3 \pm 0.2$  & $99.3 \pm 0.2$  & $99.1 \pm 0.3$ \\ \hline
Taxi  & $98.9 \pm 0.1$  & $98.9 \pm 0.1$  & $98.9 \pm 0.1$  & $98.9 \pm 0.1$ \\ \hline
\bottomrule
\end{tabular}
\caption{Table for Coverage in \% (desired coverage 0.99).}
\label{tab:coverage:real:0.99}
\end{table}

\section{Discussion}

While our approach provides a principled framework for learning minimum-volume prediction sets with finite-sample validity, several challenges remain. First, our reliance on first-order optimization methods does not guarantee avoidance of poor local minima, particularly given the inherent nonconvexity of our loss function. While our empirical results demonstrate stable performance, exploring alternative optimization strategies—such as second-order methods or tailored regularization techniques—could enhance robustness to nonconvexity.

Second, hyperparameter selection remains a critical factor. The learning rate of the matrix model, for instance, strongly influences convergence behavior, and we observed that optimal settings vary across datasets. Developing more adaptive or automated tuning strategies could improve generalization and reduce the need for manual adjustment.

Additionally, while our framework effectively minimizes volume, aggressive volume reduction may come at the expense of conditional coverage. Although our empirical results indicate strong generalization, the formulation does not explicitly regulate coverage at the level of individual feature values. Ensuring robust performance in regions with limited training data remains an open challenge, particularly in highly heteroskedastic settings. Future work could explore strategies to balance volume minimization with improved conditional reliability.

Finally, our model is designed for a fixed confidence level $\alpha$, similar to quantile regression approaches. However, one advantage of our framework is that it allows for post-hoc adjustments, enabling adaptation to different coverage levels after training. Investigating how to efficiently recalibrate prediction sets for varying confidence levels could further enhance flexibility.

Despite these challenges, our approach demonstrates strong empirical performance across both synthetic and real datasets, offering a scalable and adaptive solution to multivariate uncertainty quantification. Future research could build upon this foundation to improve theoretical guarantees, explore alternative optimization techniques, and extend the method to further application areas.

\section*{Acknowledgements}

Liviu Aolaritei acknowledges support from the Swiss National Science Foundation through the Postdoc.Mobility Fellowship (grant agreement P500PT\_222215). The remaining authors acknowledge funding from the European Union (ERC-2022-SYG-OCEAN-101071601). Views and opinions expressed are however those of the author(s) only and do not necessarily reflect those of the European Union or the European Research Council Executive Agency. Neither the European Union nor the granting authority can be held responsible for them. This publication is part of the Chair «Markets and Learning», supported by Air Liquide, BNP PARIBAS ASSET MANAGEMENT Europe, EDF, Orange and SNCF, sponsors of the Inria Foundation. This work has also received support from the French government, managed by the National Research Agency, under the France 2030 program with the reference «PR[AI]RIE-PSAI» (ANR-23-IACL-0008).

Finally, the authors would like to thank Eugène Berta, David Holzmüller, Etienne Gauthier, and Reese Pathak for their valuable feedback on the paper.

\bibliographystyle{abbrvnat} 
\bibliography{bibfile.bib}

\section*{Appendix}
\renewcommand{\thesubsection}{\Alph{subsection}}
\label{sec:appendix}

\subsection{Proof of convexity in the DC formulation}
\label{app:convexity:DC}

The DC formulation leads to the objective :
\begin{align*}
    \underbrace{- \log \det \Lambda + r \overline{\sigma}_r\left\{ \|\Lambda y_i + \eta \| \right\}}_{f(\Lambda,\eta)} 
    - \underbrace{(r-1) \overline{\sigma}_{r-1}\left\{ \|\Lambda y_i + \eta \| \right\}}_{g(\Lambda,\eta)}.
\end{align*}
To establish the convexity of $f(\Lambda, \mu)$, we analyze its two components. The first term, $-\log \det \Lambda$, is a standard convex function over the space of positive semidefinite matrices. The second term, $r \overline{\sigma}_r \{ \|\Lambda y_i + \eta\| \}$, can be shown to be convex by expressing it as the following optimization problem
\begin{align}
\label{eq:app:sum:largest:r}
    \overline{\sigma}_r\{\|\Lambda y_i + \eta\|\} = \frac{1}{r} \sum_{j=1}^r \sigma_j \{\|\Lambda y_i + \eta\|\} = \min_{\nu \in \mathbb{R}} \frac{1}{r} \sum_{i=1}^n \max\{ \|\Lambda y_i + \eta\| - \nu, 0\} + \nu,
\end{align}
where the equality follows from the first-order optimality conditions of the right-hand side. Since the norm $\|\Lambda y_i + \eta\|$ is jointly convex in $(\Lambda, \eta)$, the optimization problem on the right-hand side is jointly convex in $(\Lambda, \eta, \nu)$. Moreover, as partial minimization of a jointly convex function preserves convexity, it follows that $\overline{\sigma}_r \{ \|\Lambda y_i + \eta\| \}$ is convex. The same reasoning applies to $g(\Lambda, \eta)$, ensuring that both functions maintain convexity within the DC decomposition.

\subsection{Synthetic dataset generation}
\label{app:synthetic:datageneration}
We generate data according to the model $Y = f(X) + t(X)B$, where $X \sim \mathcal{N}(0, I_d)$ and $Y \in \mathbb{R}^k$. The noise term $B$ follows a fixed distribution and is modulated by the transformation $t(X)$. The transformation function $t(X)$ is defined as \begin{equation} t(X) = r(X) \exp \left( \sum_{i=1}^{K} w_i(X) \log R_i \right), \end{equation} where $r(X)$ scales the transformation, and the summation interpolates between $K$ randomly generated fixed rotation matrices ${R_i}$ in the Lie algebra via the log-Euclidean framework, ensuring valid rotations.

The interpolation weights $w_i(X)$ are determined based on the proximity of $X$ to $K$ anchor points in $\mathbb{R}^d$, following an inverse squared distance scheme: 
\begin{equation*} 
    \tilde{w}_i(X) = \frac{1}{\|X - X_i\|^4}, \quad w_i(X) = \frac{\tilde{w}_i(X)}{\sum_{j} \tilde{w}_j(X)}. 
\end{equation*} 
This formulation allows the covariance of $Y | X$ to vary smoothly with $X$, generating complex, structured heteroskedastic noise patterns.

The radius function is set to $r(X) = \|X\|/2 + X^\top v + 0.15$, and the function $f(X)$ is defined as \begin{equation*} f(X)= 2 (\sin (X^\top \beta) + \tanh((X \odot X)^\top \beta) + X^\top J_2), \end{equation*} where $v$ and $\beta$ are random fixed parameters in $\mathbb{R}^{d\times k}$, and $J_2$ is a coefficient matrix with all entries set to zero except for $(1,1)$ and $(2,2)$, which are set to $1$.

We generate a dataset of $30,000$ samples from this distribution, which is then split into training, validation, calibration, and test sets.

\subsection{Additional results for the synthetic datasets} 
\label{app:synthetic:additional:results}
See Tables~\ref{tab:volume:synthetic:0.99} \&~\ref{tab:coverage:synthetic:0.99}.

\subsection{Dataset information} 
\label{app:dataset:information}
See Table~\ref{tab:dataset:information}.

\subsection{Hyper-parameters}
\label{app:hyperparameters}
Each dataset requires its own hyperparameter tuning. In Table~\ref{tab:hyperparameters}, we summarize the hyperparameters used for the House dataset at a coverage level of $0.99$. The primary parameters adjusted were the learning rates and the number of training epochs. Detailed hyperparameter settings for all experiments are available in the associated GitHub repository.

\subsection{Results in one dimension for a coverage of $0.90$} 
\label{app:1D:coverage:09}
See Figure~\ref{fig:conditional:1d:09}.

\bigskip\bigskip

\begin{table}[h]
\centering
\begin{tabular}{@{}lcccc@{}}
\toprule
Dataset & Naïve QR & Emp.\ Cov.\ & Loc.\ Emp.\ Cov.\ & MVCS \\
\midrule
Exp. Fixed  &  $11.98 \pm 0.11$  & $10.05 \pm 0.08$  & $9.90 \pm 0.23$  & $\textbf{9.30} \pm 0.12$ \\ \hline
Exp. Transformed &  $18.46 \pm 0.44$  & $18.52 \pm 0.78$  & $15.23 \pm 0.47$  & $\textbf{14.46} \pm 0.46$ \\ \hline
Gau. Fixed  &  $14.07 \pm 0.70$  & $14.66 \pm 0.43$  & $10.94 \pm 0.20$  & $\textbf{10.89} \pm 0.22$ \\ \hline
Gau. Transformed  &  $9.05 \pm 0.10$  & $7.48 \pm 0.18$  & $7.09 \pm 0.12$  & $\textbf{6.74} \pm 0.13$ \\ \hline
\bottomrule
\end{tabular}
\caption{Table for Normalized Volume (desired coverage 0.99).}
\label{tab:volume:synthetic:0.99}
\end{table}

\begin{table}[h]
\centering
\begin{tabular}{@{}lcccc@{}}
\toprule
Dataset & Naïve QR & Emp.\ Cov.\ & Loc.\ Emp.\ Cov.\ & MVCS \\
\midrule
Exp. Fixed   & $98.9 \pm 0.1$  & $98.8 \pm 0.2$  & $98.8 \pm 0.2$  & $98.8 \pm 0.1$ \\ \hline
Exp. Transformed  & $99.1 \pm 0.2$  & $99.0 \pm 0.3$  & $98.8 \pm 0.2$  & $98.8 \pm 0.2$ \\ \hline
Gau. Fixed   & $99.0 \pm 0.2$  & $98.8 \pm 0.2$  & $99.0 \pm 0.2$  & $98.9 \pm 0.1$ \\ \hline
Gau. Transformed   & $99.0 \pm 0.1$  & $98.9 \pm 0.2$  & $98.9 \pm 0.1$  & $98.9 \pm 0.3$ \\ \hline
\bottomrule
\end{tabular}
\caption{Table for Coverage in \% (desired coverage 0.99).}
\label{tab:coverage:synthetic:0.99}
\end{table}

\begin{table}[h]
\centering
\begin{tabular}{@{}lcccc@{}}
\toprule
Dataset & Number of samples & Dimension of covariate & Dimension of response \\
\midrule
Bias correction & 7752 & 22 & 2 \\ \hline
Casp & 45730 & 8 & 2 \\ \hline
Energy & 768 & 8 & 2 \\ \hline
House & 21613 & 17 & 2 \\ \hline
rf1 & 9125 & 64 & 8 \\ \hline
rf2 & 9125 & 576 & 8 \\ \hline
scm1d & 9803 & 280 & 16 \\ \hline
scm20d & 8966 & 61 & 16 \\ \hline
Taxi & 61286 & 6 & 2 \\ \hline
\bottomrule
\end{tabular}
\caption{Description of the datasets.}
\label{tab:dataset:information}
\end{table}

\begin{table}[t]
\centering
\begin{tabular}{@{}lcccc@{}}
\toprule
Hyperparameter & Value \\
\midrule
Hidden dimension (center) & 256 \\
Number of hidden layers (center) & 3 \\
Activation function (center) & ReLU \\
Hidden dimension (matrix) & 256 \\
Number of hidden layers (matrix) & 3 \\
Activation function (matrix) & ReLU \\
Optimizer & Adam \\
Batch Size (warm start) & 100 \\
Number of epochs (warm start) & 500 \\
Batch size (our loss) & 100 \\
Number of epochs (matrix) & 20 \\
Number of epochs (our loss) & 200 \\
Learning rate (warm start) & 0.0001 \\
Learning rate (model) & 0.0005 \\
Learning rate (matrix) & 0.01 \\
Learning rate (p) & 0.01 \\
Number of neighbors & 1000 \\
Learning scheduler & CosineAnnealingLR \\
Batch size (pinball loss) & 100 \\
Number of epochs (pinball loss) & 200 \\
Learning rate (pinball loss) & 0.0001 \\
\bottomrule
\end{tabular}
\caption{Hyper-parameters on the house dataset - coverage level $0.99$.}
\label{tab:hyperparameters}
\end{table}

\begin{figure}[t]
    \centering
    \includegraphics[width=0.47\linewidth]{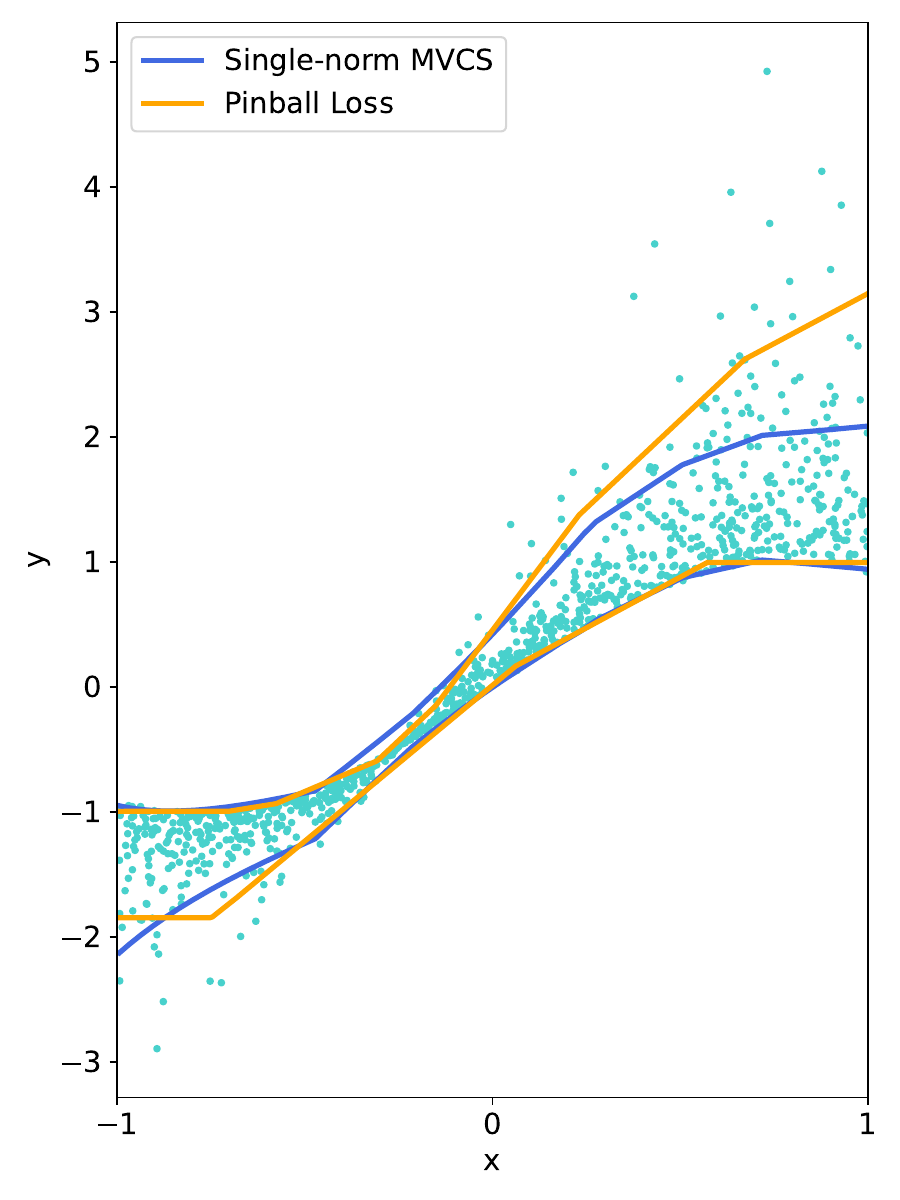}
    \vspace{-5mm}
    \caption{Conformal sets for the desired coverage $0.90$.}
    \label{fig:conditional:1d:09}
\end{figure}

\end{document}